\documentclass{siamart250211}

\usepackage{graphicx}
\graphicspath{{images/}}
\usepackage{amssymb,amsfonts,mathrsfs,bm}
\usepackage[algo2e]{algorithm2e}
\usepackage{subfigure}
\usepackage{booktabs}
\usepackage{multirow}%=
\newcommand{\R}{\mathbb{R}}
\newcommand{\norm}[1]{\left\lVert #1\right\rVert}
\newcommand{\abs}[1]{\left\lvert #1\right\rvert}

\theoremstyle{plain}
\theoremheaderfont{\normalfont\itshape}
\theorembodyfont{\normalfont}
\theoremseparator{.}
\newtheorem{remark}[theorem]{Remark}

\title{Gauss-Map Variation for Image Denoising: Geometric Analysis and an Anderson--Accelerated Majorization--Minimization Method}
\headers{H. SU}{GAUSS-MAP VARIATION FOR IMAGE DENOISING}
\author{Haibin Su\thanks{Yau Mathematical Sciences Center, Tsinghua University, Beijing 100084, China (\email{hbsu@mail.tsinghua.edu.cn}).}
}
\begin{document}

\maketitle

\begin{abstract}
We propose a Gauss-map variation (GMV) model for image denoising
that measures the spatial variation of the tangent-plane projectors
of the scaled image graph. We establish an equivalent representation of the regularizer in terms of the corresponding Gauss map and, using differential geometric tools including tubular coordinates and the Frenet frame, analyze its behavior across general $C^2$ and piecewise $C^2$ boundaries. The resulting estimates provide edge- and corner-contrast preservation properties. To solve the proposed model, we introduce a bilinear decomposition involving a unit normal field and a scalar magnitude field and develop an Anderson-accelerated majorization--minimization algorithm. The normal field subproblem admits an explicit pointwise majorization--minimization update, which is combined with an Anderson acceleration. For both $L^1$ and $L^2$ data fidelity terms, we establish sufficient decrease and boundedness of the iterates and prove that the generated sequence converges to a critical point of the penalized model. Numerical experiments on synthetic and natural images demonstrate the boundary preserving capability of the proposed model and its competitive performance in removing Gaussian and impulsive noise.
\end{abstract}


%-------------------------------------------------------------------------------------

\section{Introduction}

Image denoising aims to recover a clean image from noisy observations. Recovering fine structures is particularly challenging because noise and meaningful image features may occupy overlapping pixels. Edges, corners, and intensity contrasts encode object boundaries and shape information, and their preservation is essential to a faithful reconstruction. Variational methods address this problem by combining a data fidelity term with a regularizer that expresses prior information. The central modeling challenge is to suppress noise while retaining these geometric features, rather than treating all intensity variations as undesirable oscillations.

The total variation model \cite{LR1992} effectively removes noise while preserving sharp edges \cite{strong2003edge}. However, it may produce staircase artifacts, reduce image contrast, and smooth corners \cite{AC1997, dobson1996recovery, YM2002, bellettini2002total}. High order models \cite{AC1997, YLY2000, ML2003, KB2010} alleviate staircasing and better recover smoothly varying regions. Nevertheless, they may oversmooth sharp structures,
thereby blurring object boundaries and corners.

Geometry based models provide a complementary approach by regularizing geometric quantities associated with image level sets or the image graph. Representative examples include Euler's elastica model \cite{Shen2003euler, Tai2011Afast, Zhu2013Image, Deng2019Anew, liu2024fast, su2026accelerated}, mean curvature based model \cite{WZ2012, zhu2020image, zhao2026image}, Gaussian curvature model \cite{CBL2016, HL2022}, and Weingarten-map based model \cite{QZ2022}. Related curvature based regularizers have also been investigated in
\cite{zhong2021image, wu2025total, zhang2025edge, lu2026total}. The total scaled-gradient variation (TSGV) framework \cite{wu2024variational, su2026cone} instead penalizes the variation of a nonlinearly scaled image gradient. These models incorporate geometric information, such as local orientation and bending, that is not fully represented by derivative magnitudes alone. This geometric dependence allows them to accommodate sharp object boundaries while suppressing oscillations in smooth regions.

The preservation of edges, corners, and intensity contrast has also motivated analytical studies of geometric regularizers. The authors in \cite{WZ2012} investigated these properties for a mean curvature based model. Related analyses were subsequently developed for Gaussian curvature \cite{CBL2016} and Weingarten-map based \cite{QZ2022} models. The authors in \cite{wu2024variational} established the same results for TSGV under suitable assumptions on the scaling function. A common strategy in these analyses is to compare a sharp piecewise constant image with a family of smooth transition profiles. Under appropriate conditions, the limiting regularization cost of a sharp transition is independent of its jump height, which helps explain why these models can preserve intensity contrast.

However, the scope of these conclusions depends on the geometric classes considered in the corresponding proofs. Existing analyses are commonly established for radially symmetric disk images or idealized rectangular objects. Although these geometries make the relevant integrals tractable, the resulting conclusions do not directly extend to boundaries with spatially varying curvature. This limitation motivates an analysis for general piecewise smooth boundaries.

In this paper, we propose a Gauss-map variation (GMV) model based
on the tangent plane geometry of the scaled image graph. For a
sufficiently smooth image $u$, let
\[
m(u)
=
\frac{(-\alpha\nabla u,1)^\top}
{\sqrt{1+\alpha^2|\nabla u|_2^2}},
\qquad
P(u)=I_3-m(u)m(u)^\top,
\]
where \(\alpha>0\) balances the spatial coordinates and image
intensity. Here, \(m(u)\) is the unit normal to the scaled image
graph, and \(P(u)\) is the orthogonal projector onto its tangent
plane. We define the GMV regularizer by
\[
\mathcal R(u)
:=
\int_\Omega
\|\nabla P(u)\|_{\mathrm F}\mathrm dx
=
\sqrt{2}
\int_\Omega
\|\nabla m(u)\|_{\mathrm F}\mathrm dx
\]
and combine it with an \(L^p\) data-fidelity term for
\(p\in\{1,2\}\). The geometric structure of this regularizer enables
the analysis across general \(C^2\) boundaries. Through suitable smooth approximations, the analysis is further extended to piecewise \(C^2\) boundaries containing finitely many corners.

The numerical treatment of geometric regularizers is challenging
because they involve nonlinear and nonsmooth regularization terms. Augmented Lagrangian methods \cite{Tai2011Afast,WZ2013,CBL2016,QZ2022} typically introduce multiple penalty parameters, and their numerical performance is often sensitive to the choice of these parameters. Operator-splitting schemes provide efficient alternatives \cite{Deng2019Anew, HL2022, wu2024variational}. However, for more complicated nonconvex models, rigorous convergence guarantees often remain unavailable. More recently, bilinear decompositions have yielded convergent algorithms for Euler's elastica and mean curvature models \cite{liu2024fast, liu2025minimizing}. These developments motivate a tractable and convergent algorithm for the GMV model.

To solve the proposed GMV model, we introduce a bilinear
representation involving a unit normal field and a scalar magnitude
field and develop an Anderson-accelerated majorization--minimization (AAMM) algorithm for the resulting quadratic penalty formulation. The normal field subproblem admits an explicit MM update, which is further accelerated by a depth-one Anderson correction. We further prove that the entire sequence generated by the proposed algorithm converges to a critical point of the constrained penalized objective.

The main contributions are summarized as follows:
\begin{itemize}

\item
We introduce an image denoising regularizer that measures the
variation of the tangent plane projectors of the scaled image graph. We prove that this projector based regularizer is equivalent to the variation of the corresponding Gauss map. This equivalence provides a clear geometric interpretation and clarifies the connections with TSGV and curvature based regularization.

\item
Using differential geometric tools including tubular coordinates and the Frenet frame, we analyze the regularization cost across general $C^2$ boundaries and extend the analysis to piecewise $C^2$ boundaries. For the approximation classes considered in this paper, the limiting cost is proportional to the boundary length and is independent of the jump height. These results yield edge- and corner-contrast preservation properties.

\item
We develop a bilinear alternating algorithm in which the constrained normal field subproblem admits an explicit MM update and the remaining variables can be updated efficiently. A depth-one
Anderson correction accelerates the MM iteration while preserving a sufficient decrease condition. For both $p=1$ and $p=2$, we prove that the generated sequence has finite length and converges to a critical point of the  penalized discrete model. Numerical experiments further demonstrate that the proposed method is competitive with representative geometry based methods for Gaussian noise removal and impulsive noise removal, while effectively
preserving geometrically significant image structures.
\end{itemize}

The remainder of this paper is organized as follows. Section~\ref{sec:Gauss_map_variation_model} introduces the GMV model and its geometric interpretation. Section~\ref{sec:Equivalent_formulation} establishes an equivalent
formulation and discusses its relationship to other second order regularizers. Section~\ref{sec:Edge_preservation} analyzes edge- and corner-contrast preservation for $C^2$ smooth and piecewise smooth boundaries. Section~\ref{sec:numerical_algorithm} presents
the discrete bilinear formulation and the AAMM algorithm, and Section~\ref{sec:convergence} establishes its convergence properties. Numerical experiments are reported in Section~\ref{sec:experiments}. Section~\ref{sec:conclusion} concludes the paper.

\section{A Gauss-map variation model}\label{sec:Gauss_map_variation_model}
Let $\Omega\subset\R^2$ be a bounded open domain. We denote by
$u:\Omega\to\R$ the unknown grayscale image and by
$f:\Omega\to\R$ the observed image. To describe the local geometry
of $u$, we embed it into $\R^3$ as a scaled graph.

\subsection{The image graph and its tangent-plane projection}
Let $\alpha>0$ be a scaling parameter that balances the spatial
coordinates and the image intensity. The corresponding scaled
image graph is defined by
\[
\Sigma
:=
\left\{
\bigl(x_1,x_2,\alpha u(x_1,x_2)\bigr):
(x_1,x_2)\in\Omega
\right\}
\subset\R^3.
\]
A natural parametrization of $\Sigma$ is given by
$X_u(x_1,x_2)=\bigl(x_1,x_2,\alpha u(x_1,x_2)\bigr)$. The two tangent vectors are
\begin{equation*}
\partial_{x_1}X_u
=
\begin{pmatrix}
1\\
0\\
\alpha u_{x_1}
\end{pmatrix},
\qquad
\partial_{x_2}X_u
=
\begin{pmatrix}
0\\
1\\
\alpha u_{x_2}
\end{pmatrix}.
\end{equation*}
We collect them as the columns of
\begin{equation*}
B(u):=
\begin{pmatrix}
1&0\\
0&1\\
\alpha u_{x_1}&\alpha u_{x_2}
\end{pmatrix}.
\end{equation*}
Thus, for $x=(x_1,x_2)$, the range of $B(u)(x)$ is the tangent plane
$T_{X_u(x)}\Sigma$. We define $P(u)$ as
\begin{equation*}
P(u):=
B(u)\bigl(B(u)^\top B(u)\bigr)^{-1}B(u)^\top.
\end{equation*}
Then, we can deduce the following proposition.

\begin{proposition}\label{prop:properties-of-P}
Let $u\in C^{1}(\Omega)$, and define $G(u):=B(u)^{\top}B(u)$. Then, the following statements hold.

\begin{enumerate}
\item The matrix $G(u)$ is symmetric positive definite,
and hence is invertible. Moreover,
\[
G(u)=I_{2}+\alpha^{2}\nabla u\nabla u^{\top}.
\]

\item The matrix $P(u)$ is symmetric and idempotent, i.e, $P(u)^{\top}=P(u)$ and $P(u)^{2}=P(u)$. The range of $P(u)(x)$ is the tangent plane to the graph $X_u(x_1,x_2)$. In particular,
\[
P(u)B(u)=B(u).
\]

\item Define the unit normal field of the scaled graph by
\[
m(u):=
\frac{
\bigl(-\alpha u_{x_{1}},
      -\alpha u_{x_{2}},
      1\bigr)^{\top}}{\sqrt{1+\alpha^{2}|\nabla u|^{2}}}.
\]
Then $P(u)=I_{3}-m(u)m(u)^{\top}$.

\item If, in addition, $u\in C^{2}(\Omega)$, then for
$i\in\{1,2\}$,
\begin{equation}\label{P_Pi_P}
\begin{aligned}
P(u)
\bigl(\partial_{x_i}P(u)\bigr)P(u)
&=0,\\
\bigl(I_{3}-P(u)\bigr)
\bigl(\partial_{x_i}P(u)\bigr)
\bigl(I_{3}-P(u)\bigr)
&=0.
\end{aligned}
\end{equation}
\end{enumerate}
\end{proposition}

\begin{proof}
For every nonzero $\xi\in\mathbb{R}^{2}$, one has
\[
\xi^{\top}G(u)\xi
=
\xi^{\top}B(u)^{\top}
B(u)\xi
=
|B(u)\xi|^{2}>0.
\]
Therefore, $G(u)$ is symmetric positive definite and invertible. Direct multiplication gives
\[
G(u)=
I_{2}+\alpha^{2}\nabla u\,\nabla u^{\top}.
\]

Because $G(u)$ is symmetric, its inverse is also
symmetric. Hence $P(u)^{\top}=P(u)$.
Moreover, $P(u)^{2}=P(u)$. Thus, $P(u)$ is an orthogonal projection. From the definition of $P(u)$, $\operatorname{Ran}P(u) \subseteq \operatorname{Ran}B(u)$.
On the other hand,
\[
P(u)B(u)=B(u)G(u)^{-1}
B(u)^{\top}B(u)=B(u).
\]

A direct calculation shows that
\begin{equation}\label{B_times_m}
B(u)^{\top}m(u)
=
\frac{1}{\sqrt{1+\alpha^{2}|\nabla u|^{2}}}
\begin{pmatrix}
1&0&\alpha u_{x_{1}}\\
0&1&\alpha u_{x_{2}}
\end{pmatrix}
\begin{pmatrix}
-\alpha u_{x_{1}}\\
-\alpha u_{x_{2}}\\
1
\end{pmatrix}
=0.
\end{equation}
It follows that $P(u)m(u)=0$. Since $P(u)$ has rank two, its kernel is one dimensional. Therefore, $\ker P(u)=
\operatorname{span}\{m(u)\}$. Both $P(u)$ and
$I_{3}-m(u)m(u)^{\top}$ are orthogonal
projections onto the plane orthogonal to $m(u)$.
Consequently,
\begin{equation}\label{P_alpha}
P(u)
=
I_{3}-m(u)m(u)^{\top}.
\end{equation}

Finally, suppose that $u\in C^{2}(\Omega)$. Differentiating \eqref{P_alpha} with respect to $x_i$, for $i\in\{1,2\}$, gives
\begin{equation}\label{partial_P_alpha}
\partial_{x_i}P(u)
=-\partial_{x_i}m(u)m(u)^{\top}
-m(u)\partial_{x_i}m(u)^{\top}.
\end{equation}
Alternatively, differentiating $P(u)^{2}
=P(u)$ yields
\[
\bigl(\partial_{x_i}P(u)\bigr)P(u)
+P(u)\bigl(\partial_{x_i}P(u)\bigr)
=\partial_{x_i}P(u).
\]
Multiplying this identity from the left and right by
$P(u)$ proves
\[
P(u)
\bigl(\partial_{x_i}P(u)\bigr)
P(u)=0.
\]
Multiplication by $I_{3}-P(u)$ similarly gives
\[
\bigl(I_{3}-P(u)\bigr)
\bigl(\partial_{x_i}P(u)\bigr)
\bigl(I_{3}-P(u)\bigr)
=0.
\]
\end{proof}

\begin{remark}\label{rem:geometric-meaning-dP}
By decomposing $\partial_{x_i}P(u)$ with respect
to the orthogonal splitting $\mathbb{R}^3
=\operatorname{Ran}(P(u))\oplus \operatorname{Ran}(I_3-P(u))$, \eqref{P_Pi_P} shows that
\[
\partial_{x_i}P(u)=
P(u)\partial_{x_i}P(u)(I_3-P(u))+ (I_3-P(u))\partial_{x_i}P(u)P(u).
\]
Consequently, $\partial_{x_i}P(u)$ maps tangent vectors into the normal space and normal vectors into the tangent space.
\end{remark}

We define the Gauss-map variation regularizer by
\begin{equation}\label{eq:RGGV}
\mathcal R(u):=
\int_\Omega\norm{\nabla P(u)}_\mathrm{F}\mathrm{d}x,
\end{equation}
where $\norm{\nabla P(u)}_\mathrm{F}^2:=
\norm{\partial_{x_1}P(u)}_\mathrm{F}^2+
\norm{\partial_{x_2}P(u)}_\mathrm{F}^2$. For $f\in L^p(\Omega)$, the proposed Gauss-map variation denoising model is
\begin{equation}
\inf_{u\in\mathcal X_p}
\left\{E(u):=
\lambda\mathcal R(u)+
\frac{1}{p}\int_\Omega\abs{u-f}^p\mathrm{d}x\right\},
\quad 1\leq p<+\infty,
\label{eq:GGVmodel}
\end{equation}
where $\lambda>0$ is the regularization parameter and
\begin{equation}
\mathcal X_p
:=
\left\{
u\in W^{1,2}(\Omega)\cap L^p(\Omega):
P(u)\in W^{1,1}(\Omega;\R^{3\times3})
\right\}.
\end{equation}
The mapping $x\longmapsto P(u)(x)$ is a matrix representation of the unoriented Gauss map of the image graph. Moreover, $\mathcal R(u)$ can be interpreted as the total variation of the unoriented Gauss map.

\begin{remark}
The derivative and the measure in \eqref{eq:RGGV} are taken with
respect to the original image coordinates. Thus, $\mathcal R$ is the Euclidean variation on $\Omega$ of the pulled back Gauss map. It is different from the surface energy $\int_\Sigma\abs{\mathrm{d}P}\mathrm{d}A$, where $\mathrm{d}A$ denotes the surface area element. The formulation \eqref{eq:RGGV} is useful in image processing because it is defined
directly on the pixel domain $\Omega$.
\end{remark}

\section{Equivalent formulation and relation to other second-order models}\label{sec:Equivalent_formulation}

The regularizer in \eqref{eq:RGGV} admits an equivalent
representation that eliminates the matrix inverse and is therefore
more convenient for both theoretical analysis and numerical
computation. The equivalence is established below.

\begin{proposition}\label{prop:equivalence}
Let $u\in C^2(\Omega)$. Then $\norm{\nabla P(u)}_{\mathrm F}=
\sqrt{2}\norm{\nabla m(u)}_{\mathrm F}$. Consequently,
\begin{equation}
\label{eq:regularizeridentity}
\mathcal R(u)=\sqrt{2} \int_\Omega \norm{\nabla m(u)}_{\mathrm F}\mathrm{d}x,
\end{equation}
where
$\norm{\nabla m(u)}_{\mathrm F}^2:=
\abs{\partial_{x_1}m(u)}^2+
\abs{\partial_{x_2}m(u)}^2$.
\end{proposition}

\begin{proof}
Since $m(u)$ is a unit vector field, its derivative is
orthogonal to $m(u)$. Combining this orthogonality with
\eqref{partial_P_alpha}, we obtain
\begin{align*}
\left\|
\partial_{x_i}P(u)
\right\|_{\mathrm F}^{2}
&=\left\|
\partial_{x_i}m(u)m(u)^{\top}
+m(u)\partial_{x_i}m(u)^{\top}
\right\|_{\mathrm F}^{2}\\
&=2\left|\partial_{x_i}m(u)\right|^{2}.
\end{align*}
Summing over $i=1,2$ completes the proof.
\end{proof}

The projector formulation \eqref{eq:RGGV} makes the
basis independent geometric structure of the model explicit,
whereas the equivalent formulation \eqref{eq:regularizeridentity} is more convenient for analysis and implementation. Expanding the components of $m(u)$ gives the following explicit representation:
\begin{equation}\label{eq:R_explicit}
\mathcal R(u)
=
\sqrt{2}
\int_\Omega
\Bigg(
\norm{
\nabla\left(
\frac{\alpha\nabla u}{\sqrt{1+\alpha^{2}|\nabla u|^{2}}}
\right)
}_{\mathrm F}^{2}
+
\abs{
\nabla\left(
\frac{1}{\sqrt{1+\alpha^{2}|\nabla u|^{2}}}
\right)
}^{2}
\Bigg)^{1/2}\mathrm{d}x.
\end{equation}
The denoising problem associated with the regularizer is formulated as
\begin{equation}\label{eq:E_explicit}
E(u):=
\lambda\mathcal R(u)
+\frac1p\int_\Omega |u-f|^p\mathrm dx,
\end{equation}
where $\lambda>0$ is the regularization parameter, $p=1$ is used for impulsive denoising and $p=2$ is used for Gaussian denoising.

\subsection{Relation to TSGV}

Consider the scaling function
$\psi(s):=\alpha/\sqrt{1+\alpha^2s^2}$. With this choice, the
corresponding TSGV regularizer becomes
\[
\mathcal R_{\mathrm{TSGV}}(u)
:=
\int_\Omega
\norm{
\nabla\left(
\frac{\alpha\nabla u}{\sqrt{1+\alpha^{2}|\nabla u|^{2}}}
\right)
}_{\mathrm F}\mathrm{d}x.
\]
Thus, TSGV measures the variation of the two horizontal components
of the unit normal. In contrast, Gauss-map variation measures the variation of the full unit normal and therefore also includes the contribution of its vertical component. The additional term
$\Big|\nabla\Big(\frac{1}{\sqrt{1+\alpha^{2}|\nabla u|^{2}}}\Big)\Big|^{2}$ in \eqref{eq:R_explicit} records the variation of the normalization factor and becomes significant in regions where the gradient magnitude changes rapidly.

\subsection{Relation to curvature based models}

We next compare Gauss-map variation with two representative
curvature based regularizers. For a sufficiently smooth image $u$, we consider the following mean curvature regularizer associated with the scaled graph:
\begin{equation}\label{eq:RMC}
\mathcal R_{\mathrm{MC}}(u)
:=
\frac12\int_\Omega
\abs{
\operatorname{div}\left(
\frac{\alpha\nabla u}{\sqrt{1+\alpha^{2}|\nabla u|^{2}}}
\right)
}\mathrm{d}x
=
\frac12\int_\Omega
\abs{
\operatorname{tr}\left[
\nabla\left(
\frac{\alpha\nabla u}{\sqrt{1+\alpha^{2}|\nabla u|^{2}}}
\right)
\right]
}\mathrm{d}x,
\end{equation}
where $\operatorname{div}$ and $\operatorname{tr}$ denote the
divergence and trace operators, respectively. Hence, the
mean curvature model retains only the trace of the Jacobian of the
normalized gradient, whereas Gauss-map variation measures its
full variation together with that of the normalization factor.

For the same smooth graph, the Gaussian curvature is
\[
K(u)
=
\frac{\alpha^2\det(\nabla^2u)}
{(1+\alpha^{2}|\nabla u|^2)^2}.
\]
Since the surfacearea element is $\mathrm{d}A=\sqrt{1+\alpha^{2}|\nabla u|^{2}}\mathrm{d}x$, the corresponding Gaussian curvature regularizer is
\begin{equation}\label{eq:RGC}
\mathcal R_{\mathrm{GC}}(u)
:=
\int_\Sigma
\abs{K(u)}\mathrm{d}A
=
\int_\Omega
\frac{
\alpha^2\abs{\det(\nabla^2u)}
}{
(\sqrt{1+\alpha^{2}|\nabla u|^{2}})^3
}\mathrm{d}x.
\end{equation}
By using the surface geometry of the scaled graph, the mean curvature and Gaussian curvature satisfy
\[
\left|H\right|
=
\left|\frac{k_1+k_2}{2}\right|,
\qquad
\left|K\right|
=
\left|k_1k_2\right|,
\]
where $k_1$ and $k_2$ denote the principal curvatures of the
scaled graph. To place Gauss-map variation in the same geometric
framework, we next derive the corresponding expression for the surface differential of the Gauss-map.

We use the standard differential of a smooth map between
Riemannian manifolds and its Hilbert--Schmidt norm; see
\cite{lee2006riemannian,lee2018introduction}. The relationship
between the Weingarten map and the principal curvatures can be
found in \cite{do2016differential}. Here, $d_{\Sigma}$ denotes
differentiation along the tangent directions of the scaled graph,
and we define the Weingarten map by $-d_{\Sigma}m$. Let $\tau_1$ and $\tau_2$ be the orthonormal principal directions corresponding to $k_1$ and $k_2$, respectively. Then,
\[
d_{\Sigma}m(\tau_i)=-k_i\tau_i,
\qquad i=1,2.
\]
It follows that
\[
\left\|d_{\Sigma}m\right\|_{\mathrm{HS}}
=
\sqrt{k_1^2+k_2^2}.
\]
Here, $\|\cdot\|_{\mathrm{HS}}$ denotes the Hilbert--Schmidt norm. Using the projection representation of the unoriented Gauss map, we further obtain
\[
\left\|d_{\Sigma}P\right\|_{\mathrm{HS}}
=
\sqrt{2}\,
\left\|d_{\Sigma}m\right\|_{\mathrm{HS}}
=
\sqrt{2\left(k_1^2+k_2^2\right)}.
\]
From the viewpoint of the principal curvatures, the quantity
$\sqrt{k_1^2+k_2^2}$ measures the full magnitude of the normal
variation and vanishes only when both principal curvatures are
zero. By contrast, the mean curvature magnitude may vanish when $k_1=-k_2$, whereas the Gaussian curvature magnitude vanishes whenever either principal curvature is zero. Consequently, Gauss-map variation avoids the cancellation inherent in mean curvature and remains sensitive to one-dimensional or cylindrical bending modes that are not detected by Gaussian curvature.

\section{Edge-contrast preservation for $C^2$ boundaries}
\label{sec:Edge_preservation}

In this section, we study the edge-contrast preservation property
of the Gauss-map variation model \eqref{eq:GGVmodel} for objects with
$C^2$ boundaries. Let $\Omega\subset\mathbb R^2$ be an open and
bounded domain. We consider the following class of piecewise-constant
images:
\begin{align*}
\mathcal C
:=
\bigl\{
f^1=h\chi_D \big|
h>0,\,
\overline D\subset\Omega,\,
D\text{ is an open set},\, \partial D\text{ is of class }C^2
\bigr\}.
\end{align*}
For $f^1\in\mathcal C$, we denote $\Gamma:=\partial D, L:=\mathcal H^1(\Gamma)$. Here, $\chi_D$ denotes the characteristic function of $D$, namely,
\[
\chi_D(x)
:=
\begin{cases}
1, & x\in D,\\
0, & x\in\Omega\setminus D.
\end{cases}
\]
The notation $\mathcal H^1$ denotes the one dimensional Hausdorff
measure, which coincides with arc length on a $C^1$ curve.

Since $f^1$ is discontinuous across $\Gamma$, its Gauss-map variation
cannot be evaluated directly through \eqref{eq:RGGV}. We therefore
define it through smooth approximations obtained by smoothing the jump
along the normal direction of $\Gamma$. Let $d$ be the signed distance
function to $\Gamma$, chosen to be negative inside $D$ and positive
outside $D$. Since $\Gamma$ is of class $C^2$, there exists
$\delta>0$ such that the tubular coordinates around $\Gamma$ are well
defined in $U:=\{x\in\Omega:\abs{d(x)}<\delta\}$. We introduce the following class of one-dimensional transition
functions:
\begin{equation}
\label{set_S_GGV}
\begin{aligned}
\mathcal S_\epsilon
:=
\Bigl\{
\varrho\in C^2(\mathbb R)\ \Big|\
&0\leq\varrho(t)\leq1,\,
\varrho(0)=\frac12,\,
\varrho'(t)\leq0,\,
\varrho'(t)=\varrho'(-t),\\
&\varrho''(t)\leq0
\,\text{for }t<0,\quad
\varrho''(t)\geq0
\,\text{for }t>0,\\
&\varrho(t)=1
\,\text{if }t\leq-\epsilon,\quad
\varrho(t)=0
\,\text{if }t\geq\epsilon
\Bigr\},
\end{aligned}
\end{equation}
where $\epsilon\in(0,\delta)$. Set $s:=-\varrho'(0)>0$. Since $\varrho(-\epsilon)=1$, $\varrho(\epsilon)=0$, and
$-\varrho'(t)\leq-\varrho'(0)=s$ for
$t\in[-\epsilon,\epsilon]$, we have
\[
1
=
-\int_{-\epsilon}^{\epsilon}\varrho'(t)\,\mathrm dt
\leq
2\epsilon s.
\]
Thus, $s\geq1/(2\epsilon)$. In particular, for every admissible family with $\epsilon\to0^+$, one has $s\to+\infty$.

For $\varrho\in\mathcal S_\epsilon$, define $\theta(t):=\arctan\bigl(\alpha h\varrho'(t)\bigr)$. The range of $\arctan$ is taken to be $\left(-\frac{\pi}{2},\frac{\pi}{2}\right)$. Since $\varrho'$ is even, $\theta$ is also even. Moreover, $\theta(0)=-\arctan(\alpha hs)$, and
\[
\theta'(t)
=
\frac{
\alpha h\varrho''(t)
}{
1+\alpha^2h^2\bigl(\varrho'(t)\bigr)^2
}.
\]
It follows from the assumptions on $\varrho''$ that $\theta'(t)\leq0$ for $t<0$, $\theta'(t)\geq0$ for $t>0$. Furthermore, since $\varrho'$ vanishes at $t=\pm\epsilon$, we have $\theta(-\epsilon)=\theta(\epsilon)=0$. Consequently,
\begin{equation}
\label{GGV_theta_integral}
\int_{-\epsilon}^{\epsilon}
\abs{\theta'(t)}\mathrm dt
=
-2\theta(0)
\longrightarrow
\pi
\qquad\text{as }\epsilon\to0^+.
\end{equation}

We define a smooth approximation $g$ of $f^1$ by
\begin{equation}
\label{eq:g_GGV}
g(x)
=
\begin{cases}
h\varrho(d(x)),
&x\in U,\\
h,
&x\in D\setminus U,\\
0,
&x\in\Omega\setminus(D\cup U).
\end{cases}
\end{equation}
By the construction of $\varrho$, the function $g$ is of class $C^2$. We denote by $\mathcal A(f^1)$ the collection of all functions $g$ generated by \eqref{eq:g_GGV} with $0<\epsilon<\delta$ and $\varrho\in\mathcal S_\epsilon$.

Next, we derive the local representation of the Gauss-map variation of $g$ and express its variation in tubular coordinates around $\Gamma$. Let $\Gamma$ be parametrized by arclength: $y=\gamma(\sigma)$ with $\abs{\gamma'(\sigma)}=1$. Along $\Gamma$, we introduce the Frenet frame $\{\tau(\sigma),\mathrm n(\sigma)\}$, where $\tau(\sigma):=\gamma'(\sigma)$ is the unit tangent vector and $\mathrm n(\sigma)$ is the outer unit normal vector. The signed curvature $\kappa(\sigma)$ is characterized by
\[
\mathrm n'(\sigma)
=
-\kappa(\sigma)\tau(\sigma).
\]
Every point $x$ satisfying $\abs{d(x)}<\epsilon$ can be represented
uniquely as
\[
x=y+t\mathrm n(y),
\qquad
y\in\Gamma,
\quad
\abs{t}<\epsilon.
\]
Equivalently, the tubular coordinate map is
\[
X(\sigma,t)
:=
\gamma(\sigma)+t\mathrm n(\sigma).
\]

The following lemma gives the local representation of the Gauss-map
variation in these coordinates.

\begin{lemma}
\label{lem:local_representation_GGV}
Under the above setting, choose $\delta>0$ sufficiently small. Then
\begin{equation}\label{GGV_pointwise_norm}
\norm{\nabla P(g)}_{\mathrm F}
=
\sqrt{2}
\left[
\abs{\theta'(t)}^2
+
\left(
\frac{\kappa(y)\sin\theta(t)}
{1-\kappa(y)t}
\right)^2
\right]^{1/2}.
\end{equation}
Furthermore,
\begin{equation}\label{eq:R_g_GGV}
\begin{aligned}
\mathcal R(g)
&=
\sqrt{2}
\int_\Gamma
\int_{-\epsilon}^{\epsilon}
\left[
\abs{\theta'(t)}^2
+
\left(
\frac{\kappa(y)\sin\theta(t)}
{1-\kappa(y)t}
\right)^2
\right]^{1/2}\\
&\hspace{45mm}\times
\bigl(1-\kappa(y)t\bigr)
\mathrm dt\mathrm d\mathcal H^1(y),
\end{aligned}
\end{equation}
where $\kappa(y):=\kappa(\sigma)$ for $y=\gamma(\sigma)$.
\end{lemma}

\begin{proof}
Differentiating $X$ with respect to $\sigma$ and $t$ gives
\begin{equation}
\label{Differentiating_X_GGV}
\begin{aligned}
\partial_\sigma X(\sigma,t)
&=
\gamma'(\sigma)+t\mathrm n'(\sigma)
=
\bigl(1-\kappa(\sigma)t\bigr)\tau(\sigma),\\
\partial_tX(\sigma,t)
&=
\mathrm n(\sigma).
\end{aligned}
\end{equation}
By the definition of the signed distance function $d$ and
\eqref{Differentiating_X_GGV}, we have
$d(X(\sigma,t))=t$, and
\begin{align*}
\partial_t [d(X(\sigma,t))]
&=
\nabla d(X(\sigma,t))
\cdot\partial_tX(\sigma,t)
=
1,\\
\partial_\sigma [d(X(\sigma,t))]
&=
\nabla d(X(\sigma,t))
\cdot\partial_\sigma X(\sigma,t)
=
0.
\end{align*}
That is,
\[
\nabla d(X(\sigma,t))\cdot\mathrm n(\sigma)=1,
\quad
\nabla d(X(\sigma,t))\cdot\tau(\sigma)=0.
\]
Therefore, $\nabla d(X(\sigma,t))=\mathrm n(\sigma)$. By \eqref{Differentiating_X_GGV}, and taking $\delta>0$ sufficiently small, the arclength element of the parallel curve $t=\mathrm{const}$ is
\[
\abs{\partial_\sigma X(\sigma,t)}\mathrm d\sigma
=
\bigl(1-\kappa(\sigma)t\bigr)\mathrm d\sigma.
\]
Therefore, differentiation along the unit tangent and unit normal
directions is given, respectively, by
\begin{align*}
D_\tau
&=
\frac{1}{\abs{\partial_\sigma X(\sigma,t)}}
\frac{\mathrm d}{\mathrm d\sigma}
=
\frac{1}{1-\kappa(\sigma)t}
\frac{\mathrm d}{\mathrm d\sigma},\\
D_{\mathrm n}
&=
\frac{1}{\abs{\partial_tX(\sigma,t)}}
\frac{\mathrm d}{\mathrm dt}
=
\frac{\mathrm d}{\mathrm dt}.
\end{align*}

Since $g(x)=h\varrho(d(x))$ in the transition layer, we have $\nabla g=h\varrho'(d)\nabla d$. Evaluating this identity at $X(\sigma,t)$ gives $\nabla g(X(\sigma,t))=h\varrho'(t)\mathrm n(\sigma)$, and
\[
W(g)(X(\sigma,t))
=
\sqrt{
1+\alpha^2h^2\bigl(\varrho'(t)\bigr)^2
}.
\]
By the definition of $\theta$,
\[
\sin\theta(t)
=
\frac{
\alpha h\varrho'(t)
}{
\sqrt{
1+\alpha^2h^2\bigl(\varrho'(t)\bigr)^2
}
},
\quad
\cos\theta(t)
=
\frac{1}{
\sqrt{
1+\alpha^2h^2\bigl(\varrho'(t)\bigr)^2
}
}.
\]
Using the definition of $m(g)$, we obtain
\[
m(g)(X(\sigma,t))
=
\begin{pmatrix}
-\sin\theta(t)\mathrm n(\sigma)\\
\cos\theta(t)
\end{pmatrix}.
\]
Because $\{\mathrm n,\tau\}$ is an orthonormal frame in the image
domain, we have
\[
\norm{\nabla m(g)}_{\mathrm F}^2
=
\abs{D_{\mathrm n}m(g)}^2
+
\abs{D_\tau m(g)}^2,
\]
and
\begin{align*}
\abs{D_{\mathrm n}m(g)}^2
&=
\bigl(\cos^2\theta(t)+\sin^2\theta(t)\bigr)
\abs{\theta'(t)}^2
=
\abs{\theta'(t)}^2,\\
\abs{D_\tau m(g)}^2
&=
\left(
\frac{\kappa(y)\sin\theta(t)}
{1-\kappa(y)t}
\right)^2.
\end{align*}
Therefore,
\[
\norm{\nabla m(g)}_{\mathrm F}^2
=
\abs{\theta'(t)}^2
+
\left(
\frac{\kappa(y)\sin\theta(t)}
{1-\kappa(y)t}
\right)^2.
\]
Then, we obtain \eqref{GGV_pointwise_norm}.

The area element in the tubular coordinates is
\begin{equation}\label{dx_GGV}
\begin{aligned}
\mathrm dx
&=
\left|
\det\bigl(
\partial_\sigma X(\sigma,t),
\partial_tX(\sigma,t)
\bigr)
\right|
\mathrm d\sigma\mathrm dt\\
&=
\bigl(1-\kappa(\sigma)t\bigr)
\mathrm d\sigma\mathrm dt,\\
&=
\bigl(1-\kappa(y)t\bigr)
\mathrm dt\mathrm d\mathcal H^1(y).
\end{aligned}
\end{equation}
Taking $\delta>0$ sufficiently small ensures that $1-\kappa(y)t>0$.

Finally, $g$ is constant outside the transition layer $\{\abs{d}<\epsilon\}$. Hence $P(g)$ is constant there and
$\nabla P(g)=0$. Combining \eqref{GGV_pointwise_norm} and \eqref{dx_GGV} proves \eqref{eq:R_g_GGV}.
\end{proof}

The following lemmas and theorem establish the edge-contrast
preservation property of the Gauss-map variation model.

\begin{lemma}
\label{lem:generalized_edge_regularizer_GGV}
Let $f^1\in\mathcal C$. In the sense of the smooth normal
approximations introduced above, we have
\begin{equation}
\label{GGV_sharp_edge_value}
\mathcal R(f^1)
:=
\lim_{\epsilon\to0^+}\mathcal R(g)
=
\sqrt2\pi L.
\end{equation}
In particular, the limiting Gauss-map variation is independent of the
object intensity $h$ and the scaling parameter $\alpha>0$.
\end{lemma}

\begin{proof}
We first establish the lower bound. By \eqref{eq:R_g_GGV},
\begin{align*}
\mathcal R(g)
&\geq
\sqrt2
\int_\Gamma
\int_{-\epsilon}^{\epsilon}
\abs{\theta'(t)}
\bigl(1-\kappa(y)t\bigr)
\mathrm dt\mathrm d\mathcal H^1(y).
\end{align*}
Since $\theta$ is even, $\abs{\theta'}$ is even. Therefore,
$t\abs{\theta'(t)}$ is odd and $\int_{-\epsilon}^{\epsilon}
t\abs{\theta'(t)}\mathrm dt=0$. It follows that
\begin{equation}
\label{GGV_regularizer_lower_bound}
\mathcal R(g)
\geq
\sqrt2 L
\int_{-\epsilon}^{\epsilon}
\abs{\theta'(t)}\mathrm dt
=2\sqrt2 L\arctan(\alpha hs)
\longrightarrow
\sqrt2\pi L
\quad\text{as }\epsilon\to0^+.
\end{equation}

For the upper bound, \eqref{eq:R_g_GGV} gives
\begin{equation}
\label{GGV_regularizer_upper_bound}
\begin{aligned}
\mathcal R(g)
&\leq
\sqrt2
\int_\Gamma
\int_{-\epsilon}^{\epsilon}
\abs{\theta'(t)}
\bigl(1-\kappa(y)t\bigr)
\mathrm dt\mathrm d\mathcal H^1(y)\\
&\quad+
\sqrt2
\int_\Gamma
\int_{-\epsilon}^{\epsilon}
\abs{\kappa(y)}
\abs{\sin\theta(t)}
\mathrm dt\mathrm d\mathcal H^1(y).
\end{aligned}
\end{equation}
As above, the first term satisfies
\[
\sqrt2
\int_\Gamma
\int_{-\epsilon}^{\epsilon}
\abs{\theta'(t)}
\bigl(1-\kappa(y)t\bigr)
\mathrm dt\mathrm d\mathcal H^1(y)
\longrightarrow
\sqrt2\pi L
\quad\text{as }\epsilon\to0^+.
\]
For the second term, $\abs{\sin\theta(t)}\leq1$ yields
\[
\sqrt2
\int_\Gamma
\int_{-\epsilon}^{\epsilon}
\abs{\kappa(y)}
\abs{\sin\theta(t)}
\mathrm dt\mathrm d\mathcal H^1(y)
\leq
2\sqrt2\epsilon
\int_\Gamma
\abs{\kappa(y)}
\mathrm d\mathcal H^1(y)
\longrightarrow0
\quad\text{as }\epsilon\to0^+.
\]
Consequently,
\[
\limsup_{\epsilon\to0^+}\mathcal R(g)
\leq
\sqrt2\pi L.
\]
Together with \eqref{GGV_regularizer_lower_bound}, this proves
\eqref{GGV_sharp_edge_value}.
\end{proof}

\begin{lemma}
\label{lem:generalized_edge_fidelity_GGV}
Let $f^1\in\mathcal C$ and $p\in[1,+\infty)$. Then every
$g\in\mathcal A(f^1)$ satisfies
\begin{equation}
\label{GGV_fidelity_lower_bound}
\frac1p
\int_\Omega
\abs{g-f^1}^p\mathrm dx
\geq
\frac{Lh^p}{2^p p(p+1)s}.
\end{equation}
\end{lemma}

\begin{proof}
Since $g(x)=h\varrho(d(x))$ in the transition layer, we have
\[
\abs{g-f^1}
=
\begin{cases}
h\bigl(1-\varrho(d(x))\bigr),
&-\epsilon<d(x)<0,\\
h\varrho(d(x)),
&0<d(x)<\epsilon.
\end{cases}
\]
The tangent line to the graph of $\varrho$ at $t=0$ is $T(t)=\frac12-st$. It intersects the levels $1$ and $0$ at $t=-1/(2s)$ and $t=1/(2s)$, respectively. Since $s\geq1/(2\epsilon)$, we have $\left[-\frac1{2s},\frac1{2s}\right] \subset[-\epsilon,\epsilon]$. Since $\varrho''(t)\leq0$ for $t<0$, the function $\varrho$ lies below its tangent line $\varrho(t)\leq T(t)$, for $-\frac1{2s}\leq t\leq0$. Similarly, since $\varrho''(t)\geq0$ for $t>0$, the function $\varrho$ lies above its tangent line $\varrho(t)\geq T(t)$, for $0\leq t\leq\frac1{2s}$. It follows that
\begin{equation}
\label{GGV_fidelity_tubular_bound}
\begin{aligned}
\frac1p
\int_\Omega
\abs{g-f^1}^p\mathrm dx
&\geq
\frac{h^p}{p}
\int_\Gamma
\int_{-\frac1{2s}}^0
\bigl(1-T(t)\bigr)^p
\bigl(1-\kappa(y)t\bigr)
\mathrm dt\mathrm d\mathcal H^1(y)\\
&\quad+
\frac{h^p}{p}
\int_\Gamma
\int_0^{\frac1{2s}}
T(t)^p
\bigl(1-\kappa(y)t\bigr)
\mathrm dt\mathrm d\mathcal H^1(y).
\end{aligned}
\end{equation}

For each fixed $y\in\Gamma$, the sum of the two inner integrals is
\begin{equation*}
\int_{-\frac1{2s}}^0
s^p\left(t+\frac1{2s}\right)^p
\bigl(1-\kappa(y)t\bigr)\mathrm dt
+
\int_0^{\frac1{2s}}
s^p\left(\frac1{2s}-t\right)^p
\bigl(1-\kappa(y)t\bigr)\mathrm dt.
\end{equation*}
Setting $\xi=t+\frac1{2s}$, in the first integral and $\xi=\frac1{2s}-t$ in the second one, the above expression becomes
\begin{align*}
&\int_0^{\frac1{2s}}
s^p\xi^p
\left[
1-\kappa(y)
\left(\xi-\frac1{2s}\right)
\right]\mathrm d\xi
+\int_0^{\frac1{2s}}
s^p\xi^p
\left[
1-\kappa(y)
\left(\frac1{2s}-\xi\right)
\right]\mathrm d\xi\\
&=
2\int_0^{\frac1{2s}}s^p\xi^p\,\mathrm d\xi
=
\frac{1}{2^p(p+1)s}.
\end{align*}
Substituting this identity into
\eqref{GGV_fidelity_tubular_bound} and using
$L=\mathcal H^1(\Gamma)$, we obtain
\[
\begin{aligned}
\frac1p
\int_\Omega
\abs{g-f^1}^p\mathrm dx
&\geq
\frac{h^p}{p}
\int_\Gamma
\frac{1}{2^p(p+1)s}
\mathrm d\mathcal H^1(y)\\
&=
\frac{Lh^p}{2^p p(p+1)s}.
\end{aligned}
\]
\end{proof}

\begin{theorem}
\label{thm:generalized_edge_GGV}
Let $f^1\in\mathcal C$ and let $p\in[1,+\infty)$. Then there exists
a positive constant $c_1$ such that, whenever
$0<\lambda<c_1$,
\begin{equation}\label{GGV_edge_minimization}
E(f^1)
=
\inf_{g\in\mathcal A(f^1)}
E(g),
\end{equation}
where $E(f^1)$ is understood in the smooth-approximation sense
of Lemma~\ref{lem:generalized_edge_regularizer_GGV}.
\end{theorem}

\begin{proof}
By \eqref{GGV_regularizer_lower_bound} and
\eqref{GGV_fidelity_lower_bound}, every
$g\in\mathcal A(f^1)$ satisfies
\begin{equation}\label{GGV_energy_lower_bound}
E(g)
\geq
2\sqrt2\lambda L\arctan(\alpha hs)
+\frac{Lh^p}{2^p p(p+1)s}.
\end{equation}
Accordingly, consider
\[
\eta(s):=
2\sqrt2\lambda L\arctan(\alpha hs)
+\frac{Lh^p}{2^p p(p+1)s},
\quad s>0.
\]
Differentiating $\eta$ gives
\begin{equation}\label{GGV_eta_derivative}
\eta'(s)=
\frac{2\sqrt2\lambda L\alpha h}{1+\alpha^2h^2s^2}
-\frac{Lh^p}{2^p p(p+1)s^2
}.
\end{equation}
Define $c_1:=\frac{\alpha h^{p+1}}{2^{p+1}\sqrt2\,p(p+1)}>0$. If $\lambda<c_1$, then $\eta'(s)<0$ for every $s>0$. Thus, $\eta$ is strictly decreasing. Moreover, by Lemma~\ref{lem:generalized_edge_regularizer_GGV},
\[
\lim_{s\to+\infty}\eta(s)
=
\sqrt2\pi\lambda L
=
E(f^1).
\]
Consequently, $\eta(s)\geq E(f^1)$. Combining this inequality with \eqref{GGV_energy_lower_bound}, we obtain
\[
\inf_{g\in\mathcal A(f^1)}
E(g)
\geq
E(f^1).
\]
It remains to establish the reverse inequality. As
$\epsilon\to0^+$, Lemma~\ref{lem:generalized_edge_regularizer_GGV}
gives
\[
\mathcal R(g)\longrightarrow\mathcal R(f^1)
=
\sqrt2\pi L.
\]
Moreover, since $\abs{g-f^1}\leq h$ and
$g=f^1$ outside the transition layer, the tubular-coordinate formula
gives
\[
\begin{aligned}
\frac1p
\int_\Omega
\abs{g-f^1}^p\mathrm dx
&\leq
\frac{h^p}{p}
\int_\Gamma
\int_{-\epsilon}^{\epsilon}
\bigl(1-\kappa(y)t\bigr)
\mathrm dt\mathrm d\mathcal H^1(y)\\
&=
\frac{2\epsilon Lh^p}{p}
\longrightarrow0.
\end{aligned}
\]
Hence, $E(g)\longrightarrow E(f^1)$ as $\epsilon\to0^+$. Therefore, $\inf_{g\in\mathcal A(f^1)}E(g)\leq E(f^1)$. Combining the two inequalities proves \eqref{GGV_edge_minimization}.
\end{proof}

Theorem~\ref{thm:generalized_edge_GGV} shows that, within the class of smooth normal transitions introduced above, the energy of the discontinuous image equals the infimum of the energies of its smooth
approximations. Moreover, the limiting regularization $\sqrt2\pi L$ depends only on the length of the object boundary and is independent of its intensity $h$. In this sense, the Gauss-map variation model preserves both the object boundary and its contrast.

\subsection{Extension to piecewise $C^2$ boundaries}
\label{sec:GGV_edge_contrast_piecewise_C2}

We now extend the preceding edge contrast preservation analysis of
the Gauss-map variation model to images whose object boundaries are
piecewise $C^2$ and contain finitely many corner points. We consider
the following class of piecewise constant images:
\begin{align*}
\mathcal C_{\mathrm{pw}}
:=
\bigl\{
f^2=h\chi_D \big|
&h>0,\,
\overline D\subset\Omega,\,
D\text{ is an open set},
\partial D\text{ is piecewise }C^2
\bigr\}.
\end{align*}
More precisely, we assume that there exist finitely many pairwise
disjoint open $C^2$ arcs $\Gamma_1,\ldots,\Gamma_M$ and finitely many corner points $\mathcal P:=\{P_1,\ldots,P_M\}$, such that
\[
\partial D
=
\bigcup_{l=1}^M\overline{\Gamma_l},
\qquad
\partial D\setminus\mathcal P
=
\bigcup_{l=1}^M\Gamma_l,
\]
and the closures of the arcs meet only at points in $\mathcal P$. The
total boundary length is denoted by $L:=\mathcal H^1(\partial D)$.

The signed distance function to $\partial D$ is generally not of
class $C^2$ near the corner points. Therefore, the tubular-coordinate
argument used for smooth boundaries cannot be applied directly to
$\partial D$. To overcome this difficulty, we approximate $D$ by a
family of domains with $C^2$ boundaries. More precisely, we assume
that there exists a family of open sets $\{D_r\}_{0<r<r_0}$ such that $\overline{D_r}\subset\Omega$, and $\partial D_r$ is of class $C^2$. The associated piecewise-constant images are defined by $f_r^2:=h\chi_{D_r}$.

We assume that the approximation modifies the domain only in
neighborhoods of the corner points:
\[
D_r\mathbin{\Delta}D
:=
(D_r\cup D)\setminus(D_r\cap D)
\subset
\bigcup_{l=1}^M B_r(P_l),
\]
where $D_r\mathbin{\Delta}D$ denotes the symmetric difference between
$D_r$ and $D$. Consequently
\[
\abs{D_r\mathbin{\Delta}D}
\leq
M\pi r^2
\longrightarrow0
\quad\text{as }r\to0^+.
\]
We further assume that the boundary lengths converge:
\[
L_r
:=
\mathcal H^1(\partial D_r)
\longrightarrow
\mathcal H^1(\partial D)
=
L \quad\text{as }r\to0^+.
\]
In addition, the total curvature of the approximating boundaries is
assumed to be uniformly bounded:
\begin{equation}
\label{uniform_total_curvature_GGV}
\sup_{0<r<r_0}
\int_{\partial D_r}
\abs{\kappa_r(y)}
\,\mathrm d\mathcal H^1(y)
<
+\infty,
\end{equation}
where $\kappa_r$ denotes the signed curvature of $\partial D_r$.

Let $d_r$ be the signed distance function to $\partial D_r$, chosen
to be negative inside $D_r$ and positive outside $D_r$. Since
$\partial D_r$ is of class $C^2$, for every fixed $r>0$ there exists
$\epsilon_r>0$ such that $d_r$ is of class $C^2$ in the tubular
neighborhood $U_r:=\{x\in\Omega:\abs{d_r(x)}<\epsilon_r\}$. For each fixed $r>0$, choose $\varrho_r\in\mathcal S_{\epsilon_r}$, where
$\mathcal S_{\epsilon_r}$ is defined in \eqref{set_S_GGV}. Set $s_r:=-\varrho_r'(0)\geq 1/(2\epsilon_r)$. We define a smooth approximation $g_r$ of $f_r^2$ by
\begin{equation}
\label{eq:g_piecewise_GGV}
g_r(x)
=
\begin{cases}
h,
&d_r(x)\leq-\epsilon_r,\\[1mm]
h\varrho_r(d_r(x)),
&\abs{d_r(x)}<\epsilon_r,\\[1mm]
0,
&d_r(x)\geq\epsilon_r.
\end{cases}
\end{equation}
We denote by $\mathcal B(f^2)$ the collection of all admissible
families $\{g_r\}_{0<r<r_0}$, generated by \eqref{eq:g_piecewise_GGV} such that $\epsilon_r\to0^+$ as $r\to0^+$.

We now establish the corresponding edge-contrast preservation results for piecewise $C^2$ boundaries. The proofs are based on the smooth boundary results and the above approximations.

\begin{lemma}
\label{lem:generalized_edge_regularizer_GGV_pw}
Let $f^2\in\mathcal C_{\mathrm{pw}}$. For every admissible family
$\{g_r\}_{0<r<r_0}\in\mathcal B(f^2)$ satisfying
$\epsilon_r\to0$ as $r\to0^+$, we have
\begin{equation}\label{GGV_piecewise_regularizer_limit}
\mathcal R(f^2)
:=
\lim_{r\to0^+}\mathcal R(g_r)
=
\sqrt2\pi L.
\end{equation}
In particular, the limiting Gauss-map variation is independent of the
object intensity $h$ and the scaling parameter $\alpha$.
\end{lemma}

\begin{proof}
For each fixed $r>0$, the boundary $\partial D_r$ is of class $C^2$.
Applying the estimates obtained in the proof of
Lemma~\ref{lem:generalized_edge_regularizer_GGV}, with
$D$, $L$, $\epsilon$, $\kappa$, and $s$ replaced by
$D_r$, $L_r$, $\epsilon_r$, $\kappa_r$, and $s_r$, respectively,
gives
\begin{align*}
2\sqrt2 L_r\arctan(\alpha hs_r)
&\leq
\mathcal R(g_r)
\notag\\
&\leq
2\sqrt2 L_r\arctan(\alpha hs_r)
+
2\sqrt2\epsilon_r
\int_{\partial D_r}
\abs{\kappa_r(y)}
\mathrm d\mathcal H^1(y).
\end{align*}
The result follows from $s_r\geq\frac{1}{2\epsilon_r}\to+\infty$ as $r\to0^+$, and \eqref{uniform_total_curvature_GGV}.
\end{proof}

We can now establish the edge-contrast preservation result for
piecewise $C^2$ boundaries. For convenience in the subsequent analysis, we define
\[
E_f(g)
:=
\lambda\mathcal R(g)
+
\frac{1}{p}
\int_\Omega |g-f|^p\mathrm dx.
\]

\begin{theorem}
\label{thm:generalized_edge_GGV_pw}
Let $f^2\in\mathcal C_{\mathrm{pw}}$ and
$p\in[1,+\infty)$. For the constant $c_1$ in Theorem~\ref{thm:generalized_edge_GGV} whenever $0<\lambda<c_1$,
\begin{equation}
\label{GGV_piecewise_edge_minimization}
E_{f^2}(f^2)
=
\inf_{\{g_r\}_{0<r<r_0}\in\mathcal B(f^2)}
\liminf_{r\to0^+}
E_{f^2}(g_r),
\end{equation}
where $E_{f^2}(f^2)$ is understood in the smooth-approximation sense
of Lemma~\ref{lem:generalized_edge_regularizer_GGV_pw}.
\end{theorem}

\begin{proof}
Let $\{g_r\}_{0<r<r_0}\in\mathcal B(f^2)$ be arbitrary. For every fixed $r>0$, we have $f_r^2\in\mathcal C$ and $g_r\in\mathcal A(f_r^2)$. Applying Theorem~\ref{thm:generalized_edge_GGV} to $f_r^2$ yields
\begin{equation}
\label{GGV_piecewise_smooth_energy_bound}
E_{f_r^2}(g_r)
=
\lambda\mathcal R(g_r)
+
\frac1p
\int_\Omega
\abs{g_r-f_r^2}^p\mathrm dx
\geq
E_{f_r^2}(f_r^2)
=
\sqrt2\pi\lambda L_r.
\end{equation}
Using the reverse triangle inequality, we have
\[
\frac1p
\int_\Omega
\abs{g_r-f^2}^p\mathrm dx
-
\frac1p
\int_\Omega
\abs{g_r-f_r^2}^p\mathrm dx
\longrightarrow0.
\]
It follows that $E_{f^2}(g_r)-E_{f_r^2}(g_r)\to0$. Using \eqref{GGV_piecewise_smooth_energy_bound} and
$L_r\to L$, we obtain $\liminf_{r\to0^+}E_{f^2}(g_r)\geq\sqrt2\pi\lambda L$. Taking the infimum over all admissible families and using Lemma~\ref{lem:generalized_edge_regularizer_GGV_pw} gives
\begin{equation}
\label{GGV_piecewise_energy_lower_bound}
\inf_{\{g_r\}_{0<r<r_0}\in\mathcal B(f^2)}
\liminf_{r\to0^+}E_{f^2}(g_r)
\geq
\sqrt2\pi\lambda L
=E_{f^2}(f^2).
\end{equation}
By Lemma~\ref{lem:generalized_edge_regularizer_GGV_pw}, $\mathcal R(g_r)\to\sqrt2\pi L$. Furthermore,
\[
\abs{g_r-f^2}^p
\leq
2^{p-1}
\left(
\abs{g_r-f_r^2}^p
+
\abs{f_r^2-f^2}^p
\right).
\]
The functions $g_r$ and $f_r^2$ differ only in the transition layer
$\{\abs{d_r}<\epsilon_r\}$. Therefore, the tubular-coordinate formula
gives
\[
\begin{aligned}
\int_\Omega
\abs{g_r-f_r^2}^p\mathrm dx
&\leq
h^p\abs{U_r}\\
&=h^p
\int_{\partial D_r}
\int_{-\epsilon_r}^{\epsilon_r}
\bigl(1-\kappa_r(y)t\bigr)
\mathrm dt\mathrm d\mathcal H^1(y)\\
&=
2\epsilon_rL_rh^p.
\end{aligned}
\]
And,
\[
\int_\Omega
\abs{f_r^2-f^2}^p\mathrm dx
=
h^p\abs{D_r\mathbin{\Delta}D}
\leq
M\pi r^2h^p.
\]
Consequently,
\[
\int_\Omega
\abs{g_r-f^2}^p\mathrm dx
\leq
2^{p-1}h^p
\left(
2\epsilon_rL_r+M\pi r^2
\right)
\longrightarrow0.
\]
Hence, $E_{f^2}(g_r)\to\sqrt2\pi\lambda L$. It follows that
\begin{equation}
\label{GGV_piecewise_energy_upper_bound}
\inf_{\{g_r\}_{0<r<r_0}\in\mathcal B(f^2)}
\liminf_{r\to0^+}E_{f^2}(g_r)
\leq
\sqrt2\pi\lambda L
=E_{f^2}(f^2).
\end{equation}
Combining
\eqref{GGV_piecewise_energy_lower_bound} and
\eqref{GGV_piecewise_energy_upper_bound} proves
\eqref{GGV_piecewise_edge_minimization}.
\end{proof}

\section{Numerical Algorithm}\label{sec:numerical_algorithm}

\subsection{Discrete formulation}\label{subsec:discrete_formulation}

We first introduce the discrete setting used in this section.
Let the rectangular image domain be partitioned into an
\(N\times N\) Cartesian grid, $\Omega:=\{(i,j):1\leq i,j \leq N\}$, with spatial mesh sizes $\delta=1$. We introduce the discrete spaces
\[
\mathbb X:=\mathbb R^{N\times N},
\quad
\mathbb Y:=\mathbb X^2,
\quad
\mathbb Z:=\mathbb X^3,
\quad
\mathbb T:=\mathbb X^{3\times2},
\]
for scalar images, two-dimensional vector fields,
three-dimensional vector fields, and \(3\times2\) matrix fields,
respectively. The inner products on these spaces are defined componentwise using
\[
\langle U,V\rangle
:=
\sum_{i,j=1}^{N}
\langle U_{i,j},V_{i,j}\rangle,
\]
where the pointwise inner product is understood as the scalar,
Euclidean, or Frobenius inner product, as appropriate. For $v\in\mathbb X^d$ and $\boldsymbol A\in\mathbb T$, we define
\[
\|v\|_p
:=
\left(
\sum_{i,j=1}^{N}
|v_{i,j}|^p
\right)^{1/p},
\quad 1\leq p<\infty,
\]
and
\[
\|\boldsymbol A\|_{2,1}
:=
\sum_{i,j=1}^{N}
\|\boldsymbol A_{i,j}\|_{\mathrm F},
\quad
\|\boldsymbol A\|_{2,\infty}
:=
\max_{1\leq i,j\leq N}
\|\boldsymbol A_{i,j}\|_{\mathrm F}.
\]

For \(u\in\mathbb X\), the forward and backward differences subject
to periodic boundary conditions are defined as follows:
\begin{align*}
(D_1^+u)_{i,j}
&=
\begin{cases}
u_{i+1,j}-u_{i,j}, & 1\leq i<N,\\
u_{1,j}-u_{N,j},   & i=N,
\end{cases}
\qquad
(D_2^+u)_{i,j}
=
\begin{cases}
u_{i,j+1}-u_{i,j}, & 1\leq j<N,\\
u_{i,1}-u_{i,N},   & j=N,
\end{cases}\\
(D_1^-u)_{i,j}
&=
\begin{cases}
u_{i,j}-u_{i-1,j}, & 2\leq i\leq N,\\
u_{1,j}-u_{N,j},   & i=1,
\end{cases}
\qquad
(D_2^-u)_{i,j}
=
\begin{cases}
u_{i,j}-u_{i,j-1}, & 2\leq j\leq N,\\
u_{i,1}-u_{i,N},   & j=1,
\end{cases}
\end{align*}
for \(1\leq i,j\leq N\).
The discrete gradient $\nabla:\mathbb X\longrightarrow\mathbb Y$, is defined by
\[
(\nabla u)_{i,j}
:=
\begin{pmatrix}
(D_1^+u)_{i,j}\\
(D_2^+u)_{i,j}
\end{pmatrix}.
\]
For a three-dimensional vector field
$\boldsymbol n=(n_1,n_2,n_3)^\top\in\mathbb Z$,
the spatial difference operator
$\bar\nabla:\mathbb Z\longrightarrow\mathbb T$
is defined componentwise by
\[
(\bar\nabla\boldsymbol n)_{i,j}
:=
\begin{pmatrix}
(D_1^+n_1)_{i,j} & (D_2^+n_1)_{i,j}\\
(D_1^+n_2)_{i,j} & (D_2^+n_2)_{i,j}\\
(D_1^+n_3)_{i,j} & (D_2^+n_3)_{i,j}
\end{pmatrix}.
\]
To describe the corresponding adjoints,
\[
(D_1^+)^\ast=-D_1^-,
\quad
(D_2^+)^\ast=-D_2^-.
\]
Consequently, for
\(v=(v^1,v^2)^\top\in\mathbb Y\),
\[
\nabla^\ast v
=
-D_1^-v^1-D_2^-v^2.
\]
Similarly, for \(\boldsymbol A\in\mathbb T\),
\[
\bar{\nabla}^{\ast}\boldsymbol A
=
-\begin{pmatrix}
D_1^-A^{11}+D_2^-A^{12}\\
D_1^-A^{21}+D_2^-A^{22}\\
D_1^-A^{31}+D_2^-A^{32}
\end{pmatrix}.
\]

\subsection{An Anderson-accelerated majorization--minimization method}\label{sec:aos_ggv}

We develop an Anderson-accelerated majorization--minimization (AAMM) algorithm based on a bilinear reformulation of the GMV denoising model
\begin{equation}
\label{eq:aos_ggv_model}
\min_u
E(u)
:=
\frac1p\|u-f\|_p^p
+
\sqrt{2}\lambda
\|\bar\nabla\mathcal M(\nabla u)\|_{2,1},
\end{equation}
where $p\in\{1,2\}$, $\lambda,\alpha>0$, and
\[
\mathcal M(\nabla u)
=
\frac{
(-\alpha D_1^+u,-\alpha D_2^+u,1)^\top
}{
\sqrt{1+\alpha^2|\nabla u|_2^2}
}.
\]
Introduce a scalar field $q\in\mathbb X$ and a three dimensional vector field $\boldsymbol n\in\mathbb Z$ through
\[
q=\sqrt{1+\alpha^2|\nabla u|^2},
\qquad
\boldsymbol n=\mathcal M(\nabla u).
\]
These variables satisfy the bilinear identity
\[
q\odot\boldsymbol n
=
\begin{pmatrix}
-\alpha\nabla u\\
1
\end{pmatrix},
\]
where $\odot$ denotes pointwise multiplication of a scalar field and a vector field. Define the closed upper hemisphere by
\[
\mathbb S_+^2
:=
\left\{
\boldsymbol a=(a_1,a_2,a_3)^\top\in\mathbb R^3:
|\boldsymbol a|_2=1,\ a_3\ge0
\right\}.
\]
Consequently, for $1\le i,j\le N$, \eqref{eq:aos_ggv_model} is equivalent to the constrained problem
\begin{equation}\label{eq:bilinear_constrained_model}
\begin{aligned}
\min_{u,\boldsymbol n,q}\quad&
\frac1p\|u-f\|_p^p
+\sqrt2\lambda\|\bar\nabla\boldsymbol n\|_{2,1},\\
\text{s.t.}\quad&
q\odot\boldsymbol n
=
\begin{pmatrix}
-\alpha\nabla u\\
1
\end{pmatrix},
\, \boldsymbol n_{i,j}\in\mathbb S_+^2,
\, q_{i,j}\ge1.
\end{aligned}
\end{equation}

To construct an alternating algorithm, we smooth
the regularization term and penalize the bilinear
constraint. For fixed $\varepsilon>0$ and $\rho>0$,
consider
\begin{equation}\label{eq:bilinear_penalty_model}
\begin{aligned}
\min_{\substack{
u, \boldsymbol n, q \\
\boldsymbol n_{i,j}\in\mathbb S_+^2, q_{i,j}\ge1}}
\quad
E_{\rho,\varepsilon}(u,\boldsymbol n,q)
:=&
\frac1p\|u-f\|_p^p
+\sqrt2\lambda
\sum_{i,j}
\sqrt{
\|(\bar\nabla\boldsymbol n)_{i,j}\|_{\mathrm F}^2
+\varepsilon
}\\
&+
\frac{\rho}{2}
\left\|
q\odot\boldsymbol n
-
\begin{pmatrix}
-\alpha\nabla u\\
1
\end{pmatrix}
\right\|_2^2.
\end{aligned}
\end{equation}
The constrained formulation \eqref{eq:bilinear_constrained_model} is exactly equivalent to the original discrete model. For fixed $\rho$ and $\varepsilon$, \eqref{eq:bilinear_penalty_model} is a penalized,
smoothed approximation. The $u$-update depends on
the choice of $p$, whereas the $\boldsymbol n$-
and $q$-updates are common to both cases.

\subsection{The case of $p=2$}\label{subsec:aos_ggv_p2}

\textbf{u-subproblem:} For $p=2$, fixing $\boldsymbol n^k$ and $q^k$
gives the $u$-subproblem
\begin{equation}\label{eq:bilinear_u_p2_subproblem}
u^{k+1}
=
\operatorname*{argmin}_{u}
\left\{
\frac12\|u-f\|_2^2
+
\frac{\rho}{2}
\left\|
\alpha\nabla u
+
q^k\odot
\begin{pmatrix}
n_1^k\\
n_2^k
\end{pmatrix}
\right\|_2^2
\right\}.
\end{equation}
The term involving $q^kn_3^k-1$ is independent
of $u$ and has therefore been omitted.
The objective in
\eqref{eq:bilinear_u_p2_subproblem}
is strongly convex, and its unique minimizer satisfies
\begin{equation}\label{eq:bilinear_u_p2_system}
\left(
I+\rho\alpha^2\nabla^\ast\nabla
\right)u^{k+1}
=
f-\rho\alpha\nabla^\ast
\left(
q^k\odot
\begin{pmatrix}
n_1^k\\
n_2^k
\end{pmatrix}
\right).
\end{equation}
Under periodic boundary conditions, this linear system can be solved using the fast Fourier transform (FFT).

\textbf{$\mathbf n$-subproblem:} For the $\mathbf n$-subproblem, define
\[
\phi(\mathbf n)
=
\sqrt{2}\lambda
\sum_{i,j}
\sqrt{
\|(\bar\nabla\mathbf n)_{i,j}\|_{\mathrm F}^2
+\varepsilon
},
\quad
h^k(\mathbf n)
=
\frac{\rho}{2}
\sum_{i,j}
\left|
q_{i,j}^k\mathbf n_{i,j}
-
\begin{pmatrix}
-\alpha(\nabla u^{k+1})_{i,j}\\
1
\end{pmatrix}
\right|^2.
\]
After omitting terms independent of $\mathbf n$,
the subproblem is
\begin{equation}
\label{n-subproblem}
\min_{\mathbf n}
\left\{
\phi(\mathbf n)+h^k(\mathbf n)
+
\sum_{i,j}
\mathbb I_{\mathbb S_+^2}(\mathbf n_{i,j})
\right\}.
\end{equation}
The exact expansion of $h^k$ is
\begin{equation}
\label{second_order_h}
h^k(\mathbf n)
=
h^k(\mathbf n^k)
+
\left\langle
\nabla_{\mathbf n}h^k(\mathbf n^k),
\mathbf n-\mathbf n^k
\right\rangle
+
\frac{\rho}{2}
\sum_{i,j}
(q_{i,j}^k)^2
|\mathbf n_{i,j}-\mathbf n_{i,j}^k|^2.
\end{equation}
To construct a quadratic majorizer of $\phi$, define
\[
w_{i,j}^k
=
\frac{\sqrt{2}\lambda}{
\sqrt{
\|(\bar\nabla\mathbf n^k)_{i,j}\|_{\mathrm F}^2
+\varepsilon
}
},
\qquad
\ell_{i,j}^k=2w_{i,j}^k+w_{i-1,j}^k+w_{i,j-1}^k.
\]
By the concavity of the square-root function, we have
\begin{equation}\label{second_order_phi}
\begin{aligned}
\phi(\mathbf n)
\leq&
\phi(\mathbf n^k)
+
\frac12\sum_{i,j}w_{i,j}^k
\left(
\|(\bar\nabla\mathbf n)_{i,j}\|_{\mathrm F}^2
-
\|(\bar\nabla\mathbf n^k)_{i,j}\|_{\mathrm F}^2
\right)\\
=&
\phi(\mathbf n^k)
+
\left\langle
\nabla_{\mathbf n}\phi(\mathbf n^k),
\mathbf n-\mathbf n^k
\right\rangle
+
\frac12\sum_{i,j}w_{i,j}^k
\|(\bar\nabla(\mathbf n-\mathbf n^k))_{i,j}\|_{\mathrm F}^2\\
\leq&
\phi(\mathbf n^k)
+
\left\langle
\nabla_{\mathbf n}\phi(\mathbf n^k),
\mathbf n-\mathbf n^k
\right\rangle
+
\sum_{i,j}\ell_{i,j}^k
|\mathbf n_{i,j}-\mathbf n_{i,j}^k|_2^2,
\end{aligned}
\end{equation}
where the last inequality is obtained by using $|\mathbf x-\mathbf y|_2^2\leq2|\mathbf x|_2^2+2|\mathbf y|_2^2$. Combining the exact expansion \eqref{second_order_h} with the upper bound \eqref{second_order_phi}, we obtain the MM surrogate
\begin{equation}
\label{Qn}
\begin{aligned}
Q(\mathbf n\mid\mathbf n^k)
:={}&
(\phi+h^k)(\mathbf n^k)
+
\left\langle
\nabla_{\mathbf n}\phi(\mathbf n^k)
+
\nabla_{\mathbf n}h^k(\mathbf n^k),
\mathbf n-\mathbf n^k
\right\rangle\\
&+
\frac12\sum_{i,j}
\left(
\rho(q_{i,j}^k)^2+2\ell_{i,j}^k
\right)
|\mathbf n_{i,j}-\mathbf n_{i,j}^k|_2^2
+
\sum_{i,j}
\mathbb I_{\mathbb S_+^2}(\mathbf n_{i,j}),
\end{aligned}
\end{equation}
which satisfies
\begin{equation}\label{majorization_and_tangency_properties}
\begin{aligned}
Q(\mathbf n\mid\mathbf n^k)
\geq&
\phi(\mathbf n)+h^k(\mathbf n)
+
\sum_{i,j}
\mathbb I_{\mathbb S_+^2}(\mathbf n_{i,j}),\\
Q(\mathbf n^k\mid\mathbf n^k)
=&
\phi(\mathbf n^k)+h^k(\mathbf n^k).
\end{aligned}
\end{equation}
For feasible $\mathbf n$ and $\mathbf n^k$, the unit length constraints imply
\[
|\mathbf n_{i,j}-\mathbf n_{i,j}^k|_2^2
=
2-2\mathbf n_{i,j}^k\cdot\mathbf n_{i,j}
=
-2\mathbf n_{i,j}^k\cdot
(\mathbf n_{i,j}-\mathbf n_{i,j}^k).
\]
Substituting this identity and the expression for
$\nabla_{\mathbf n}h^k(\mathbf n^k)$ into \eqref{Qn}
gives
\[
Q(\mathbf n\mid\mathbf n^k)
=
Q(\mathbf n^k\mid\mathbf n^k)
-
\sum_{i,j}
\mathbf b_{i,j}^k\cdot
(\mathbf n_{i,j}-\mathbf n_{i,j}^k),
\]
where
\begin{equation}
\label{eq:n_MM_coefficients}
\mathbf b_{i,j}^k
=
2\ell_{i,j}^k\mathbf n_{i,j}^k
-
\bigl[\nabla_{\mathbf n}\phi(\mathbf n^k)\bigr]_{i,j}
+
\rho q_{i,j}^k
\begin{pmatrix}
-\alpha(\nabla u^{k+1})_{i,j}\\
1
\end{pmatrix}.
\end{equation}
Using the periodic difference operators and the
definition of $\ell^k$, we obtain
\[
\begin{aligned}
&2\ell_{i,j}^k\mathbf n_{i,j}^k
-
\bigl[\nabla_{\mathbf n}\phi(\mathbf n^k)\bigr]_{i,j}\\
=&
\ell_{i,j}^k\mathbf n_{i,j}^k
+
w_{i,j}^k
\left(
\mathbf n_{i+1,j}^k+\mathbf n_{i,j+1}^k
\right)
+
w_{i-1,j}^k\mathbf n_{i-1,j}^k
+
w_{i,j-1}^k\mathbf n_{i,j-1}^k.
\end{aligned}
\]
All coefficients on the right hand side are nonnegative, and each $\boldsymbol n_{i,j}^k$ belongs to $\mathbb S_+^2$. Hence, the third component of this sum is nonnegative. Together with $q_{i,j}^k\geq1$ and $\rho>0$, this gives
\[
(\mathbf b_{i,j}^k)_3
\geq\rho q_{i,j}^k\geq\rho>0.
\]
Since the constraints are imposed pointwise,
minimizing $Q(\cdot\mid\mathbf n^k)$ reduces to
solving
\[
\operatorname*{argmin}_{\boldsymbol n}
Q(\boldsymbol n\mid\boldsymbol n^k)
=
\operatorname*{argmin}_{\substack{
\boldsymbol n\\
\boldsymbol n_{i,j}\in\mathbb S_+^2
}}
\left\{
-\sum_{i,j}
\boldsymbol b_{i,j}^k\cdot\boldsymbol n_{i,j}
\right\}.
\]
For every feasible $\mathbf n_{i,j}$, the
Cauchy--Schwarz inequality yields
\[
-\mathbf b_{i,j}^k\cdot\mathbf n_{i,j}
\geq
-|\mathbf b_{i,j}^k||\mathbf n_{i,j}|
=
-|\mathbf b_{i,j}^k|.
\]
Because $\mathbf b_{i,j}^k\neq0$, equality holds
if and only if
$\mathbf n_{i,j}=\mathbf b_{i,j}^k/|\mathbf b_{i,j}^k|$.
This vector is feasible, since its third component
is strictly positive.
Therefore, the MM surrogate has the unique minimizer
\begin{equation}
\label{eq:n_MM_update}
(\mathbf n_{\mathrm{MM}}^{k+1})_{i,j}
=
\frac{\mathbf b_{i,j}^k}{|\mathbf b_{i,j}^k|}.
\end{equation}
For $k\geq0$, define $\mathbf r^k:=\mathbf n_{\mathrm{MM}}^{k+1}-\mathbf n^k$. \eqref{majorization_and_tangency_properties} together with \eqref{eq:n_MM_update}, imply
\begin{equation}\label{eq:n_MM_descent}
\begin{aligned}
(\phi+h^k)(\mathbf n^k)
-
(\phi+h^k)(\mathbf n_{\mathrm{MM}}^{k+1})
\geq&
Q(\mathbf n^k\mid\mathbf n^k)
-
Q(\mathbf n_{\mathrm{MM}}^{k+1}\mid\mathbf n^k)\\
=&
\sum_{i,j}\mathbf b_{i,j}^k\cdot\mathbf r_{i,j}^k
=
\frac12\sum_{i,j}
|\mathbf b_{i,j}^k||\mathbf r_{i,j}^k|^2.
\end{aligned}
\end{equation}

We then apply the MM update with a depth-one Anderson acceleration.
For $k\geq1$, we evaluate the MM mapping at $\mathbf n^{k-1}$ using the current $u^{k+1}$ and $q^k$:
\[
(\mathbf n_{\mathrm H}^k)_{i,j}
=
\Pi_{\mathbb S_+^2}\Biggl(
2\ell_{i,j}^{k-1}\mathbf n_{i,j}^{k-1}
-
\bigl[
\nabla_{\mathbf n}\phi(\mathbf n^{k-1})
\bigr]_{i,j}
+
\rho q_{i,j}^k
\begin{pmatrix}
-\alpha(\nabla u^{k+1})_{i,j}\\
1
\end{pmatrix}
\Biggr).
\]
By the same positivity argument used before, the vector being projected has a strictly positive third component. We define the residual
\[
\Delta\mathbf r^k
:=
\mathbf r^k
-
\left(
\mathbf n_{\mathrm H}^k-\mathbf n^{k-1}
\right).
\]
%Given $\eta\in(0,1)$, we damp this coefficient to
%limit the extrapolation:
%\[
%c_k
%\leftarrow
%c_k\min\left\{
%1,\,
%\frac{
%\eta\|\mathbf r^k\|_2
%}{
%\max\left\{
%|c_k|
%\|\mathbf n_{\mathrm{MM}}^{k+1}
%-\mathbf n_{\mathrm H}^k\|_2,\,
%\epsilon_{\mathrm{mach}}
%\right\}
%}
%\right\},
%\]
%where $\epsilon_{\mathrm{mach}}$ denotes machine precision.
%The notation $c_k$ subsequently refers to the
%damped coefficient.
The pointwise projected Anderson candidate is defined by
\begin{equation}
\label{eq:n_AA_projection}
(\mathbf n_{\mathrm{AA}}^{k+1})_{i,j}
=
\Pi_{\mathbb S_+^2}
\left(
\left(
\mathbf n_{\mathrm{MM}}^{k+1}
-
c_k
\left(
\mathbf n_{\mathrm{MM}}^{k+1}
-\mathbf n_{\mathrm H}^k
\right)
\right)_{i,j}
\right),
\end{equation}
where
\[
c_k
=
\begin{cases}
\dfrac{
\langle\Delta\boldsymbol r^k,\boldsymbol r^k\rangle
}{
\|\Delta\boldsymbol r^k\|_2^2
},
& \|\Delta\boldsymbol r^k\|_2>0,\\[6pt]
0,
& \|\Delta\boldsymbol r^k\|_2=0.
\end{cases}
\]
is the extrapolation coefficient. The projected point is accepted only if
\begin{equation}
\label{eq:n_AA_acceptance}
\begin{aligned}
\|\mathbf n_{\mathrm{AA}}^{k+1}
-\mathbf n_{\mathrm{MM}}^{k+1}\|_2^2
&\leq
\eta^2\|\mathbf r^k\|_2^2,\\
(\phi+h^k)(\mathbf n^k)
-
(\phi+h^k)(\mathbf n_{\mathrm{AA}}^{k+1})
&\geq
\frac12
\sum_{i,j}
|\mathbf b_{i,j}^k|
|\mathbf r_{i,j}^k|^2.
\end{aligned}
\end{equation}
The first condition controls the deviation of the Anderson
candidate from the MM update and ensures that the actual step
is comparable. Together with the energy condition, it yields a sufficient descent estimate proportional to the squared norm of the
accepted update. If \eqref{eq:n_AA_acceptance} holds, we set $\mathbf n^{k+1}=\mathbf n_{\mathrm{AA}}^{k+1}$;
otherwise, we set $\mathbf n^{k+1}=\mathbf n_{\mathrm{MM}}^{k+1}$.
At $k=0$, the MM candidate is used directly.

\textbf{q-subproblem:} the $ q$-subproblem is
\[
q^{k+1}
=
\operatorname*{argmin}_{q_{i,j}\ge1}
\frac{\rho}{2}
\left\|
q\odot\boldsymbol n^{k+1}
-
\begin{pmatrix}
-\alpha\nabla u^{k+1}\\
1
\end{pmatrix}
\right\|_2^2.
\]
This problem is separable over the pixels. Its unique solution is
\begin{equation}\label{eq:bilinear_q_update}
q_{i,j}^{k+1}
=
\max\left\{
1,\,
\boldsymbol n_{i,j}^{k+1}\cdot
\begin{pmatrix}
-\alpha(\nabla u^{k+1})_{i,j}\\
1
\end{pmatrix}
\right\}.
\end{equation}

\subsection{The case of $p=1$}
\label{subsec:bilinear_ggv_p1}

For $p=1$, fixing $\boldsymbol n^k$ and $q^k$
and omitting terms independent of $u$ gives
\begin{equation}
\label{eq:bilinear_u_p1_subproblem}
\min_u
\left\{
\|u-f\|_1
+
\frac{\rho}{2}
\left\|
\alpha\nabla u
+
q^k\odot
\begin{pmatrix}
n_1^k\\
n_2^k
\end{pmatrix}
\right\|_2^2
\right\}.
\end{equation}
We perform one MM update by majorizing the quadratic
term while retaining the fidelity term.
The Hessian of the quadratic term is
$\rho\alpha^2\nabla^\ast\nabla$.
Since $\|\nabla\|^2\leq8$ for the periodic differences
with unit grid spacing, we choose a fixed parameter $\gamma_u>8\rho\alpha^2$. Majorizing the quadratic term at $u^k$ and omitting terms independent of $u$ yields the MM update
\begin{equation}
\label{eq:bilinear_u_p1_mm}
\begin{aligned}
u^{k+1}
&=
\operatorname*{argmin}_u
\Biggl\{
\|u-f\|_1
+
\rho\alpha
\left\langle
\nabla^\ast
\left(
\alpha\nabla u^k
+
q^k\odot
\begin{pmatrix}
n_1^k\\
n_2^k
\end{pmatrix}
\right),
u-u^k
\right\rangle\\
&\qquad\qquad\qquad+
\frac{\gamma_u}{2}\|u-u^k\|_2^2
\Biggr\}\\
&=
\operatorname*{argmin}_u
\left\{
\|u-f\|_1
+
\frac{\gamma_u}{2}
\|u-f-\bar v^k\|_2^2
\right\},
\end{aligned}
\end{equation}
where $\bar v^k:=u^k-f-\frac{\rho\alpha}{\gamma_u}
\nabla^\ast\left(\alpha\nabla u^k+q^k\odot
\begin{pmatrix}n_1^k\\n_2^k\end{pmatrix}\right)$. The unique minimizer is given by the pointwise shrinkage formula
\begin{equation}
\label{eq:bilinear_u_p1_soft}
u_{i,j}^{k+1}
=
f_{i,j}
+
\max\left\{
0,\,
1-\frac{1}{\gamma_u|\bar v_{i,j}^k|}
\right\}
\bar v_{i,j}^k.
\end{equation}
The majorization property and the minimizing property
of $u^{k+1}$ imply
\begin{equation}
\label{eq:bilinear_u_p1_descent}
\begin{aligned}
&
E_{\rho,\varepsilon}(u^k,\boldsymbol n^k,q^k)
-
E_{\rho,\varepsilon}(u^{k+1},\boldsymbol n^k,q^k)\\
&\qquad\geq
\frac{\gamma_u-8\rho\alpha^2}{2}
\|u^{k+1}-u^k\|_2^2.
\end{aligned}
\end{equation}
Thus a single MM update provides sufficient decrease
without requiring an exact solution of
\eqref{eq:bilinear_u_p1_subproblem}. After updating $u$, we apply the same Anderson--MM update for $\boldsymbol n$ as in the case $p=2$. The variable $q$ is then updated by \eqref{eq:bilinear_q_update}.

We summarize the proposed algorithm for solving \eqref{eq:bilinear_penalty_model} in Algorithm~\ref{alg:bilinear_ggv}.

\begin{algorithm}[H]
\caption{Anderson-accelerated majorization--minimization (AAMM) for GMV denoising model}
\label{alg:bilinear_ggv}
\SetAlgoLined
\KwIn{
Noisy image $f$, parameters $\lambda,\alpha,\rho,\varepsilon>0$, $p\in\{1,2\}$, $\eta\in(0,1)$, tolerance $\mathrm{Tol}>0$, and $\gamma_u>8\rho\alpha^2$ for $p=1$
}
\KwOut{$u^k$}

\textbf{Initialization:}
$u^0$, $\boldsymbol n^0$, $q^0$

Set $k=0$

\While{the termination condition is not satisfied}{

    \eIf{$p=2$}{
        Compute $u^{k+1}$ by \eqref{eq:bilinear_u_p2_system}
        using FFT}{
        Compute $u^{k+1}$ by
        \eqref{eq:bilinear_u_p1_soft}}

    Compute $\boldsymbol n_{\mathrm{MM}}^{k+1}$ by \eqref{eq:n_MM_update}, and set
    $\boldsymbol n^{k+1}
    =\boldsymbol n_{\mathrm{MM}}^{k+1}$

    \If{$k\geq1$}{

        Compute $\boldsymbol n_{\mathrm{AA}}^{k+1}$ by \eqref{eq:n_AA_projection}

        \If{both conditions in \eqref{eq:n_AA_acceptance} hold}{
            Set
            $\boldsymbol n^{k+1}
            =\boldsymbol n_{\mathrm{AA}}^{k+1}$
        }
    }

    Compute $q^{k+1}$ by
    \eqref{eq:bilinear_q_update}

    $k\gets k+1$
}
\Return{$u^k$}
\end{algorithm}

\section{Convergence analysis}
\label{sec:convergence}

We first establish convergence of the proposed algorithm for
$p=2$ in \eqref{eq:bilinear_penalty_model}. To incorporate the constraints into the convergence analysis, define
\begin{equation}\label{eq:convergence_constrained_energy}
E_{\rho,\varepsilon,\mathbb I}(u,\boldsymbol n,q)
:=
E_{\rho,\varepsilon}(u,\boldsymbol n,q)
+
\sum_{i,j}
\mathbb I_{\mathbb S_+^2}(\boldsymbol n_{i,j})
+
\sum_{i,j}
\mathbb I_{[1,\infty)}(q_{i,j}).
\end{equation}

\begin{lemma}\label{lem:energy-descent}
For every $k\geq0$, the iterates satisfy
\begin{equation}
\label{eq:convergence_sufficient_decrease}
\begin{aligned}
&
E_{\rho,\varepsilon}(u^k,\boldsymbol n^k,q^k)
-
E_{\rho,\varepsilon}
(u^{k+1},\boldsymbol n^{k+1},q^{k+1})\\
&\qquad\geq
\frac12\|u^{k+1}-u^k\|_2^2
+
\frac{\rho}{2(1+\eta)^2}
\|\boldsymbol n^{k+1}-\boldsymbol n^k\|_2^2
+
\frac{\rho}{2}\|q^{k+1}-q^k\|_2^2.
\end{aligned}
\end{equation}
Consequently, the energy sequence $\left\{E_{\rho,\varepsilon} (u^k,\boldsymbol n^k,q^k)\right\}_{k\geq0}$ is nonincreasing and converges to a finite limit.
Moreover,
\begin{equation}
\label{eq:convergence_square_summability}
\sum_{k=0}^{\infty}
\left(
\|u^{k+1}-u^k\|_2^2
+
\|\boldsymbol n^{k+1}-\boldsymbol n^k\|_2^2
+
\|q^{k+1}-q^k\|_2^2
\right)
<\infty.
\end{equation}
\end{lemma}

\begin{proof}
With $\boldsymbol n^k$ and $q^k$ fixed, the Hessian of the $u$-subproblem is $I+\rho\alpha^2\nabla^\ast\nabla\succeq I$. Hence the subproblem is strongly convex with modulus at least one, and its exact minimizer satisfies
\[
E_{\rho,\varepsilon}(u^k,\boldsymbol n^k,q^k)
-
E_{\rho,\varepsilon}(u^{k+1},\boldsymbol n^k,q^k)
\geq
\frac12\|u^{k+1}-u^k\|_2^2.
\]
For the $\boldsymbol n$-update, \eqref{eq:n_MM_descent} and \eqref{eq:n_AA_acceptance} imply that either accepted update satisfies
\[
(\phi+h^k)(\mathbf n^k)
-
(\phi+h^k)(\mathbf n^{k+1})
\geq
\frac12\sum_{i,j}
|\mathbf b_{i,j}^k||\mathbf r_{i,j}^k|^2,
\]
and
\[
\|\mathbf n^{k+1}-\mathbf n^k\|_2
\leq
\|\mathbf n^{k+1}
-\mathbf n_{\mathrm{MM}}^{k+1}\|_2
+
\|\mathbf r^k\|_2
\leq
(1+\eta)\|\mathbf r^k\|_2.
\]
Combining the above estimates with
$|\mathbf b_{i,j}^k|\geq\rho$ yields
the sufficient descent property
\[
(\phi+h^k)(\mathbf n^k)
-
(\phi+h^k)(\mathbf n^{k+1})
\geq
\frac{\rho}{2(1+\eta)^2}
\|\mathbf n^{k+1}-\mathbf n^k\|_2^2.
\]
Since $|\boldsymbol n_{i,j}^{k+1}|_2=1$,
the Hessian of the $q$-subproblem is $\rho I$, and
\[
E_{\rho,\varepsilon}(u^{k+1},\boldsymbol n^{k+1},q^k)
-
E_{\rho,\varepsilon}
(u^{k+1},\boldsymbol n^{k+1},q^{k+1})
\geq
\frac{\rho}{2}\|q^{k+1}-q^k\|_2^2.
\]
Adding the above inequalities proves \eqref{eq:convergence_sufficient_decrease}.
All iterates are feasible, so the same decrease
estimate holds for $E_{\rho,\varepsilon,\mathbb I}$. The energy sequence is nonincreasing and bounded
below by zero, and therefore converges to a finite limit.
Summing \eqref{eq:convergence_sufficient_decrease}
over $k$ and using the positivity of its three
coefficients proves
\eqref{eq:convergence_square_summability}.
\end{proof}

\begin{lemma}\label{lem:boundedness}
The sequence $\{(u^k,\mathbf n^k,q^k)\}_{k\geq0}$ is bounded.
\end{lemma}

\begin{proof}
By Lemma~\ref{lem:energy-descent} and the nonnegativity of the terms in \eqref{eq:bilinear_penalty_model}, the fidelity and penalty terms are uniformly bounded. Thus $\{u^k\}$ is bounded. The unit-length constraints imply that $\{\boldsymbol n^k\}$ is bounded and $\|q^k\odot\boldsymbol n^k\|_2=\|q^k\|_2$. The triangle inequality therefore implies that ${q^k}$ is bounded.
\end{proof}

\begin{lemma}\label{lem:relative-error}
There exist a constant $C>0$, independent of $k$,
and a subgradient
\[
\boldsymbol g^{k+1}
=
\bigl(
g_u^{k+1},
g_{\boldsymbol n}^{k+1},
g_q^{k+1}
\bigr)
\in
\partial E_{\rho,\varepsilon,\mathbb I}
(u^{k+1},\boldsymbol n^{k+1},q^{k+1})
\]
such that
\begin{equation}
\label{eq:convergence_relative_error}
\|\boldsymbol g^{k+1}\|_2
\leq
C\left(
\|\boldsymbol n^{k+1}-\boldsymbol n^k\|_2
+
\|q^{k+1}-q^k\|_2
\right).
\end{equation}
\end{lemma}

\begin{proof}
For the $u$-component, define
\[
g_u^{k+1}
=
\nabla_u E_{\rho,\varepsilon}
(u^{k+1},\boldsymbol n^{k+1},q^{k+1}).
\]
Exact minimization of the $u$-subproblem gives $\nabla_u E_{\rho,\varepsilon}(u^{k+1},\boldsymbol n^k,q^k)=0$. By Lipschitz continuity,
\begin{equation}
\label{eq:convergence_u_relative_error}
\|g_u^{k+1}\|_2
\leq
C\left(
\|\boldsymbol n^{k+1}-\boldsymbol n^k\|_2
+
\|q^{k+1}-q^k\|_2
\right).
\end{equation}

For the $\boldsymbol n$-component,
\eqref{eq:n_MM_coefficients} gives
\[
\begin{aligned}
\bigl[
\nabla_{\boldsymbol n}E_{\rho,\varepsilon}
(u^{k+1},\boldsymbol n^k,q^k)
\bigr]_{i,j}
=
\left(
\rho(q_{i,j}^k)^2+2\ell_{i,j}^k
\right)\boldsymbol n_{i,j}^k
-\boldsymbol b_{i,j}^k.
\end{aligned}
\]
Since $\boldsymbol b_{i,j}^k=|\boldsymbol b_{i,j}^k|(\boldsymbol n_{\mathrm{MM}}^{k+1})_{i,j}$ and
$\boldsymbol r^k=\boldsymbol n_{\mathrm{MM}}^{k+1}-\boldsymbol n^k$,
the MM candidate satisfies
\begin{equation}
\label{eq:convergence_MM_stationarity}
\begin{aligned}
0=&
\bigl[
\nabla_{\boldsymbol n}E_{\rho,\varepsilon}
(u^{k+1},\boldsymbol n^k,q^k)
\bigr]_{i,j}
+
\left(
\rho(q_{i,j}^k)^2+2\ell_{i,j}^k
\right)\boldsymbol r_{i,j}^k\\
&+
\left(
|\boldsymbol b_{i,j}^k|_2
-\rho(q_{i,j}^k)^2-2\ell_{i,j}^k
\right)
(\boldsymbol n_{\mathrm{MM}}^{k+1})_{i,j}.
\end{aligned}
\end{equation}
Therefore, for $\left(|\boldsymbol b_{i,j}^k| -\rho(q_{i,j}^k)^2-2\ell_{i,j}^k
\right)\boldsymbol n_{i,j}^{k+1} \in\partial\mathbb I_{\mathbb S_+^2}
(\boldsymbol n_{i,j}^{k+1})$, define
\begin{equation}
\label{eq:convergence_n_subgradient}
(g_{\boldsymbol n}^{k+1})_{i,j}
:=
\bigl[
\nabla_{\boldsymbol n}E_{\rho,\varepsilon}
(u^{k+1},\boldsymbol n^{k+1},q^{k+1})
\bigr]_{i,j}
+
\left(
|\boldsymbol b_{i,j}^k|
-\rho(q_{i,j}^k)^2-2\ell_{i,j}^k
\right)\boldsymbol n_{i,j}^{k+1}.
\end{equation}
Subtracting \eqref{eq:convergence_MM_stationarity}
from \eqref{eq:convergence_n_subgradient} gives
\[
\begin{aligned}
(g_{\boldsymbol n}^{k+1})_{i,j}
=&
\Bigl[
\nabla_{\boldsymbol n}E_{\rho,\varepsilon}
(u^{k+1},\boldsymbol n^{k+1},q^{k+1})
-
\nabla_{\boldsymbol n}E_{\rho,\varepsilon}
(u^{k+1},\boldsymbol n^k,q^k)
\Bigr]_{i,j}\\
&-
\left(
\rho(q_{i,j}^k)^2+2\ell_{i,j}^k
\right)\boldsymbol r_{i,j}^k
+
\left(
|\boldsymbol b_{i,j}^k|
-\rho(q_{i,j}^k)^2-2\ell_{i,j}^k
\right)
\left(
\boldsymbol n_{i,j}^{k+1}
-
(\boldsymbol n_{\mathrm{MM}}^{k+1})_{i,j}
\right).
\end{aligned}
\]
The coefficient fields in the last two terms are
uniformly bounded.
Splitting the gradient difference at
$(u^{k+1},\boldsymbol n^{k+1},q^k)$
and applying the triangle inequality yields
\[
\begin{aligned}
\|g_{\boldsymbol n}^{k+1}\|_2
\leq&
\Bigl\|
\nabla_{\boldsymbol n}E_{\rho,\varepsilon}
(u^{k+1},\boldsymbol n^{k+1},q^{k+1})
-
\nabla_{\boldsymbol n}E_{\rho,\varepsilon}
(u^{k+1},\boldsymbol n^{k+1},q^k)
\Bigr\|_2\\
&+
\Bigl\|
\nabla_{\boldsymbol n}E_{\rho,\varepsilon}
(u^{k+1},\boldsymbol n^{k+1},q^k)
-
\nabla_{\boldsymbol n}E_{\rho,\varepsilon}
(u^{k+1},\boldsymbol n^k,q^k)
\Bigr\|_2\\
&+
C\left(
\|\boldsymbol r^k\|_2
+
\|\boldsymbol n^{k+1}
-\boldsymbol n_{\mathrm{MM}}^{k+1}\|_2
\right)\\
\leq&
C\Bigl(
\|q^{k+1}-q^k\|_2
+
\|\boldsymbol n^{k+1}-\boldsymbol n^k\|_2
+
\|\boldsymbol r^k\|_2
+
\|\boldsymbol n^{k+1}
-\boldsymbol n_{\mathrm{MM}}^{k+1}\|_2
\Bigr),
\end{aligned}
\]
Here $C>0$ is independent of $k$. The following inequality holds for both an accepted Anderson point and the MM fallback:
\[
\|\boldsymbol n^{k+1}
-\boldsymbol n_{\mathrm{MM}}^{k+1}\|_2
\leq
\eta\|\boldsymbol r^k\|_2.
\]
Consequently,
\[
\|\boldsymbol n^{k+1}-\boldsymbol n^k\|_2
\geq
\|\boldsymbol r^k\|_2
-
\|\boldsymbol n^{k+1}
-\boldsymbol n_{\mathrm{MM}}^{k+1}\|_2
\geq
(1-\eta)\|\boldsymbol r^k\|_2.
\]
Substituting these bounds into the preceding estimate gives
\begin{equation}
\label{eq:convergence_n_relative_error}
\|g_{\boldsymbol n}^{k+1}\|_2
\leq
C\left(
\|\boldsymbol n^{k+1}-\boldsymbol n^k\|_2
+
\|q^{k+1}-q^k\|_2
\right).
\end{equation}

Finally, exact minimization of the $q$-subproblem gives
\[
0\in
\bigl[
\nabla_q E_{\rho,\varepsilon}
(u^{k+1},\boldsymbol n^{k+1},q^{k+1})
\bigr]_{i,j}
+
\partial\mathbb I_{[1,\infty)}(q_{i,j}^{k+1})
\]
at every pixel. Thus we may take $g_q^{k+1}=0$. Combining \eqref{eq:convergence_u_relative_error} and \eqref{eq:convergence_n_relative_error} proves \eqref{eq:convergence_relative_error}.
\end{proof}

The preceding lemmas establish sufficient decrease,
boundedness of the iterates, and subgradient estimates.
Together with the KL property of
$E_{\rho,\varepsilon,\mathbb I}$ and its continuity
on the feasible set, these properties place the algorithm
within the KL convergence framework
\cite{attouch2013convergence,bolte2014proximal}.
In particular, \cite[Theorem~2.9]{attouch2013convergence}
yields finite length and convergence of the entire
sequence to a critical point.

\begin{theorem}\label{thm:convergence}
The sequence generated by the proposed algorithm has finite length
\begin{equation}
\label{eq:convergence_finite_length}
\sum_{k=0}^{\infty}
\left(
\|u^{k+1}-u^k\|_2
+
\|\mathbf n^{k+1}-\mathbf n^k\|_2
+
\|q^{k+1}-q^k\|_2
\right)
<\infty.
\end{equation}
Moreover, the sequence converges to a feasible critical point of $E_{\rho,\varepsilon,\mathbb I}$.
\end{theorem}

For $p=1$, the $u$-subproblem is updated by a single
MM step, while the $\boldsymbol n$- and $q$-updates
remain unchanged.
The resulting sufficient decrease and subgradient
estimates allow the KL convergence argument to extend
to this case, yielding the following theorem.

\begin{theorem}\label{cor:convergence_p1}
Let $p=1$ and $\gamma_u>8\rho\alpha^2$.
Under the same conditions as in the case $p=2$, with the $u$-update given by \eqref{eq:bilinear_u_p1_soft}, the generated sequence
has finite length and converges to a critical point
of $E_{\rho,\varepsilon,\mathbb I}$.
\end{theorem}

\begin{proof}
We verify the modifications required for the $u$-update, and the estimates for the $\boldsymbol n$- and $q$-updates remain valid.

Combining \eqref{eq:bilinear_u_p1_descent}
with the decrease estimates for the other two
updates gives
\[
\begin{aligned}
&
E_{\rho,\varepsilon}(u^k,\boldsymbol n^k,q^k)
-
E_{\rho,\varepsilon}
(u^{k+1},\boldsymbol n^{k+1},q^{k+1})\\
\geq&
\frac{\gamma_u-8\rho\alpha^2}{2}
\|u^{k+1}-u^k\|_2^2
+
\frac{\rho}{2(1+\eta)^2}
\|\boldsymbol n^{k+1}-\boldsymbol n^k\|_2^2
+
\frac{\rho}{2}
\|q^{k+1}-q^k\|_2^2.
\end{aligned}
\]
All three coefficients are strictly positive.
Moreover,
\[
\|u^k-f\|_2
\leq
\|u^k-f\|_1
\leq
E_{\rho,\varepsilon}(u^0,\boldsymbol n^0,q^0),
\]
so $\{u^k\}$ is bounded.
The boundedness of $\{\boldsymbol n^k\}$ and
$\{q^k\}$ follows by the same argument as in
Lemma~\ref{lem:boundedness}.

The optimality condition of the MM update is
\[
\begin{aligned}
-\rho\alpha\nabla^\ast
\left(
\alpha\nabla u^k
+
q^k\odot
\begin{pmatrix}
n_1^k\\
n_2^k
\end{pmatrix}
\right)
-\gamma_u(u^{k+1}-u^k)
\in
\partial\bigl(\|\cdot-f\|_1\bigr)(u^{k+1}).
\end{aligned}
\]
Adding the gradient of the quadratic penalty
at the new iterate therefore gives the following
$u$-component of a subgradient:
\[
\begin{aligned}
g_u^{k+1}
:={}&
\rho\alpha\nabla^\ast
\left[
\alpha\nabla(u^{k+1}-u^k)
+
q^{k+1}\odot
\begin{pmatrix}
n_1^{k+1}\\
n_2^{k+1}
\end{pmatrix}
-
q^k\odot
\begin{pmatrix}
n_1^k\\
n_2^k
\end{pmatrix}
\right]\\
&-\gamma_u(u^{k+1}-u^k).
\end{aligned}
\]
Using boundedness of $\{q^k\}$, the unit-length
constraints, and $\|\nabla\|^2\leq8$, we obtain
\[
\begin{aligned}
\|g_u^{k+1}\|_2
\leq{}&
(\gamma_u+8\rho\alpha^2)
\|u^{k+1}-u^k\|_2\\
&+
C\left(
\|\boldsymbol n^{k+1}-\boldsymbol n^k\|_2
+
\|q^{k+1}-q^k\|_2
\right),
\end{aligned}
\]
where $C$ is independent of $k$. The construction of $g_{\boldsymbol n}^{k+1}$ in Lemma~\ref{lem:relative-error} is unchanged, and exact minimization of the $q$-subproblem allows us to take $g_q^{k+1}=0$. Consequently,
\[
\boldsymbol g^{k+1}
=
(g_u^{k+1},g_{\boldsymbol n}^{k+1},g_q^{k+1})
\in
\partial E_{\rho,\varepsilon,\mathbb I}
(u^{k+1},\boldsymbol n^{k+1},q^{k+1}),
\]
with
\[
\|\boldsymbol g^{k+1}\|_2
\leq
C\left(
\|u^{k+1}-u^k\|_2
+
\|\boldsymbol n^{k+1}-\boldsymbol n^k\|_2
+
\|q^{k+1}-q^k\|_2
\right).
\]
Thus the required relative error estimate holds.

Finally, the $L^1$ fidelity term is continuous
and semialgebraic.
Hence $E_{\rho,\varepsilon,\mathbb I}$ remains
a proper lower semicontinuous KL function,
and its restriction to the feasible set is continuous.
Boundedness provides a convergent subsequence.
Together with sufficient decrease and the subgradient estimate, this verifies the hypotheses of
\cite[Theorem~2.9]{attouch2013convergence}.
The finite length and critical point convergence
conclusions follow.
\end{proof}

\section{Numerical experiments} \label{sec:experiments}
In this section, we present numerical experiments to evaluate the
restoration quality and computational efficiency of the proposed
AAMM algorithm (Algorithm~\ref{alg:bilinear_ggv}) for the
bilinearly reformulated GMV model. We first examine the boundary preserving behavior of the proposed model on synthetic images. We then investigate the monotonic decay of the energy and the effect of depth-one Anderson acceleration. Comparisons with representative geometric regularization methods are subsequently reported for Gaussian noise removal. Finally, we evaluate the proposed method for impulse noise removal.

All experiments are performed on a laptop with Windows 11 (64-bit), 32GB of RAM, an Intel Core i9-14900HX (2.20 GHz), and MATLAB R2023a.

\subsection{Experimental settings}

\begin{figure}[htbp]
	\centering
	\subfigure[Synthetic1]{
    \includegraphics[scale=0.38]{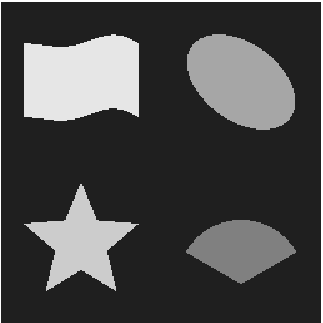}
	}
	\subfigure[Synthetic2]{
    \includegraphics[scale=0.38]{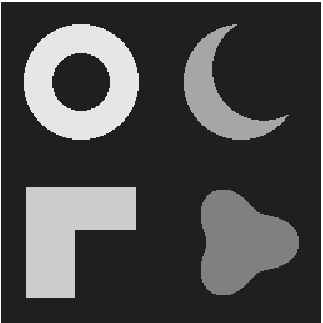}
	}
	\subfigure[Airplane]{
    \includegraphics[scale=0.38]{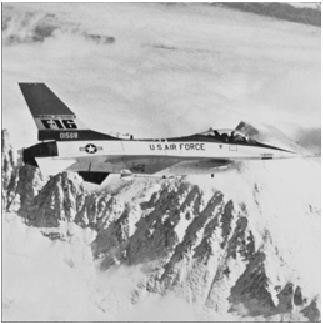}
	}
    \subfigure[Boat]{
    \includegraphics[scale=0.38]{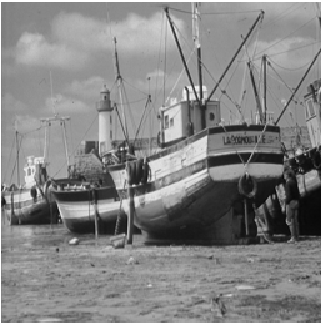}
	}
	\subfigure[Cameraman]{
    \includegraphics[scale=0.38]{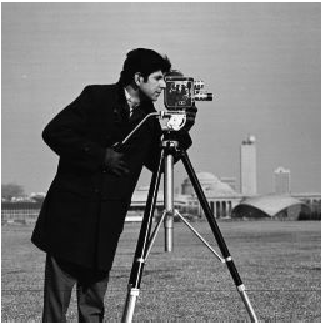}
	}\\
	\subfigure[Couple]{
    \includegraphics[scale=0.38]{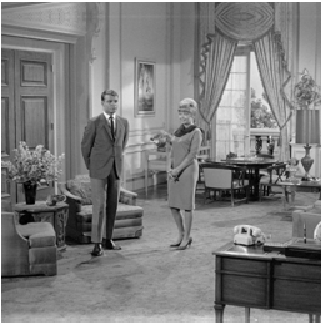}
    }
    \subfigure[Monarch]{
    \includegraphics[scale=0.38]{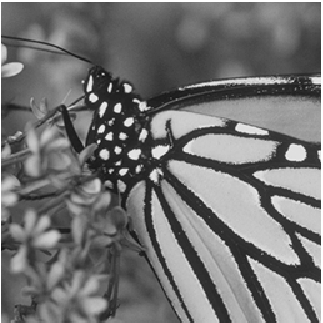}
    }
    \subfigure[Parrot]{
    \includegraphics[scale=0.38]{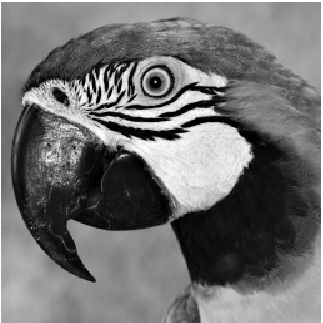}
	}
	\subfigure[Peppers]{
    \includegraphics[scale=0.38]{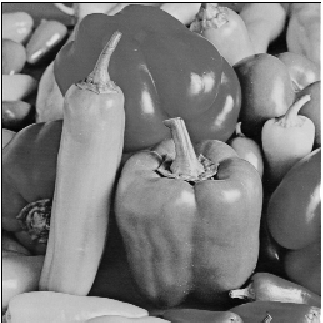}
	}
	\subfigure[Starfish]{
    \includegraphics[scale=0.38]{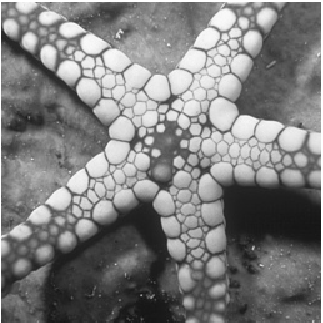}
    }
\caption{Original test images used in the denoising experiments. In the Gaussian noise experiments, Gaussian noise with mean $0$ and variance $0.0025$ is added. In the impulsive noise experiments, $5\%$ of the image pixels are randomly replaced by $0$ or $1$.} \label{figure_denoising_image}
\end{figure}

The parameter $\eta$ for the depth-one Anderson acceleration is fixed at $0.5$. The penalty parameter $\rho$, the graph scaling parameter $\alpha$, and the regularization parameter $\lambda$ are selected separately for each image. The initial image is obtained by solving $\left(I+\nabla^\ast\nabla\right)u^0=f$, and the auxiliary variables are initialized by
\[
q_{i,j}^0
=
\sqrt{
1+\alpha^2|(\nabla u^0)_{i,j}|_2^2
},
\qquad
\boldsymbol n_{i,j}^0
=
\frac{
\bigl(
-\alpha(\nabla u^0)_{i,j},1
\bigr)^\top
}{
q_{i,j}^0
},
\]
which satisfies $q_{i,j}^0\geq1$ and $\boldsymbol n_{i,j}^0\in\mathbb S_+^2$ at every pixel. The iteration is terminated when the relative error in the energy satisfies
\[
\frac{
\left|
E_{\rho,\varepsilon}
(u^{k+1},\boldsymbol n^{k+1},q^{k+1})
-
E_{\rho,\varepsilon}
(u^k,\boldsymbol n^k,q^k)
\right|
}{
\left|
E_{\rho,\varepsilon}
(u^{k+1},\boldsymbol n^{k+1},q^{k+1})
\right|
}
\leq \mathrm{Tol}.
\]
Unless otherwise specified, we set $\mathrm{Tol}=10^{-4}$ and the
maximum number of iterations to $300$. Anderson acceleration is
enabled by default.

For the competing methods, we use the parameter settings recommended
in their original papers or codes. For all methods, the model and algorithm parameters are tuned to attain the highest PSNR value. All input and reconstructed images are normalized to the intensity range $[0,1]$. Figure~\ref{figure_denoising_image} displays the synthetic images and the real images selected from the Set12 benchmark. All test images are resized to a spatial resolution of $256\times256$
pixels.

\subsection{Boundary preservation}

\begin{figure}[tbp]
    \centering
    \subfigure[Noisy]{
        \begin{minipage}[b]{0.11\linewidth}
        \centering
        \includegraphics[scale=0.265]{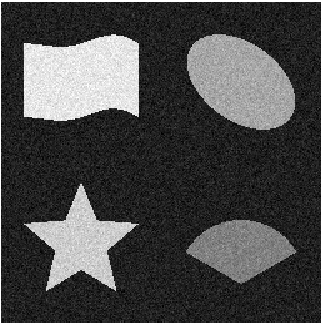}\\
        \vspace{0.02cm}
        \includegraphics[scale=0.265]{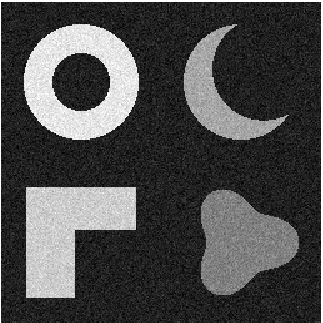}\\
        \vspace{0.02cm}
        \end{minipage}
        }
    \subfigure[\cite{WZ2012}]{
        \begin{minipage}[b]{0.11\linewidth}
        \centering
        \includegraphics[scale=0.13]{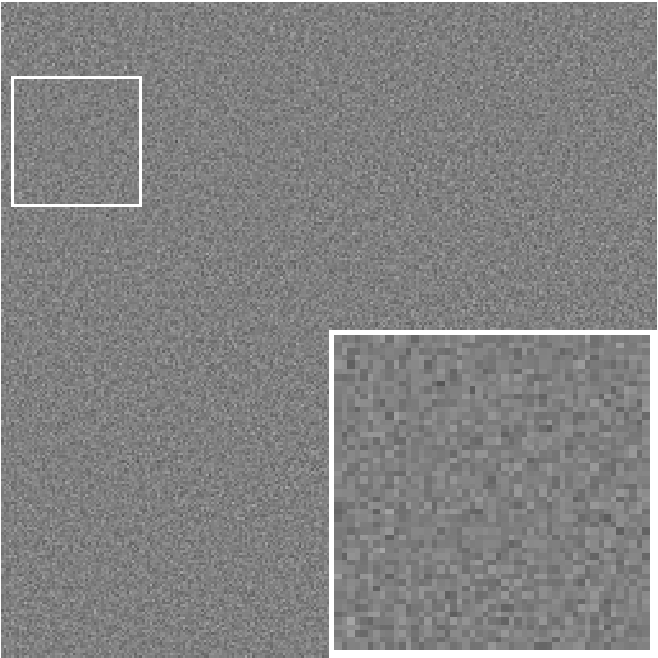}\\
        \vspace{0.02cm}
        \includegraphics[scale=0.13]{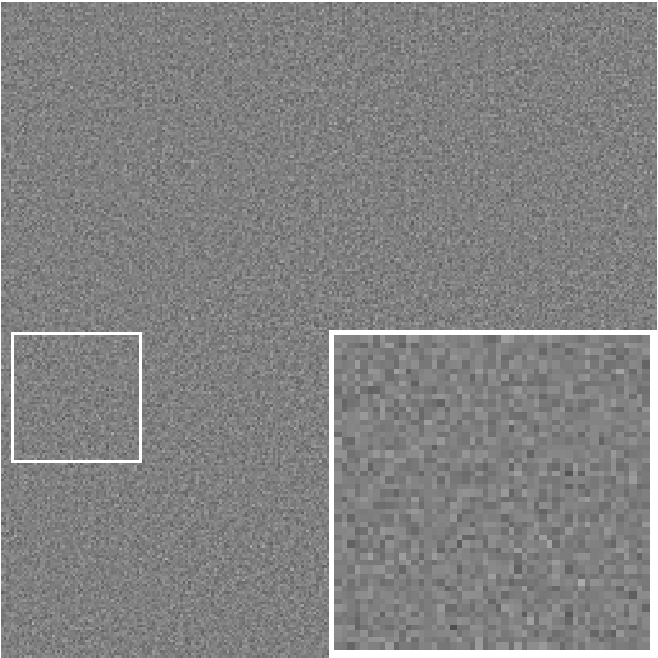}\\
        \vspace{0.02cm}
        \end{minipage}
        }
    \subfigure[\cite{CBL2016}]{
        \begin{minipage}[b]{0.11\linewidth}
        \centering
        \includegraphics[scale=0.13]{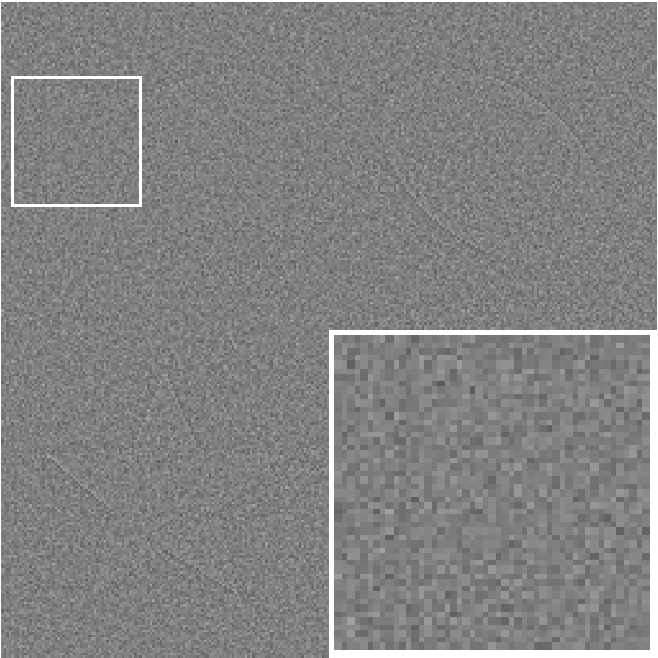}\\
        \vspace{0.02cm}
        \includegraphics[scale=0.13]{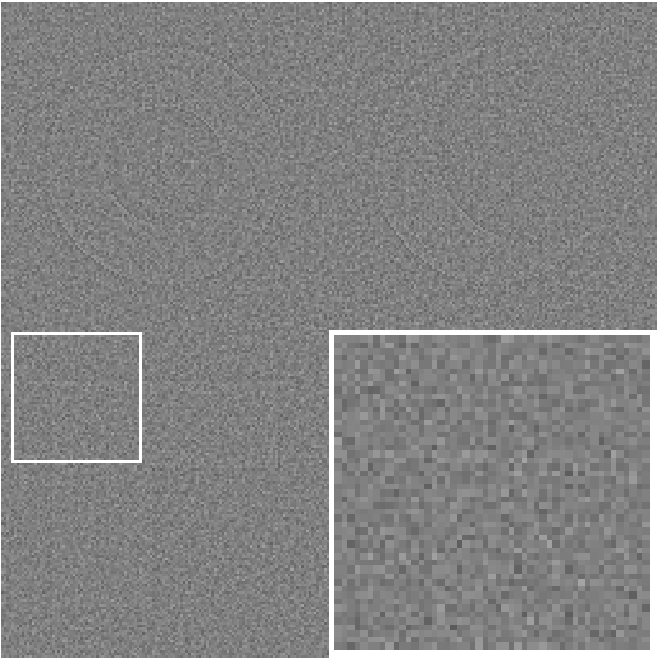}\\
        \vspace{0.02cm}
        \end{minipage}
        }
        \subfigure[\cite{QZ2022}]{
        \begin{minipage}[b]{0.11\linewidth}
        \centering
        \includegraphics[scale=0.13]{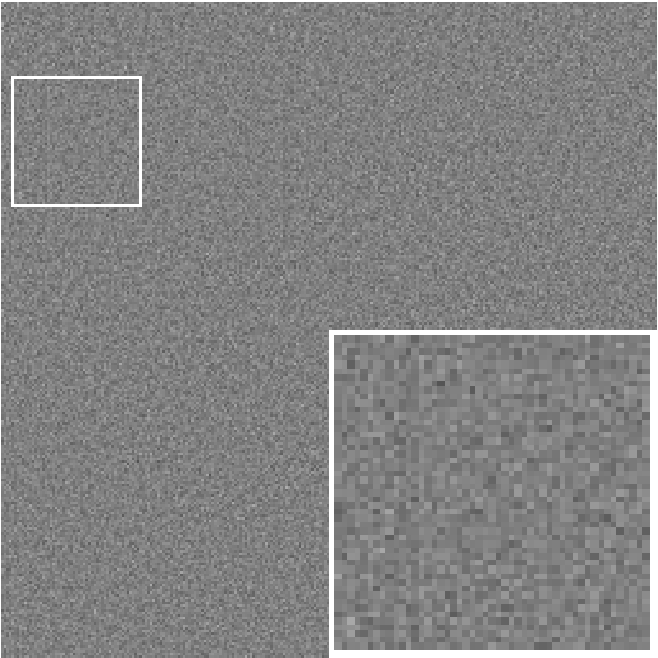}\\
        \vspace{0.02cm}
        \includegraphics[scale=0.13]{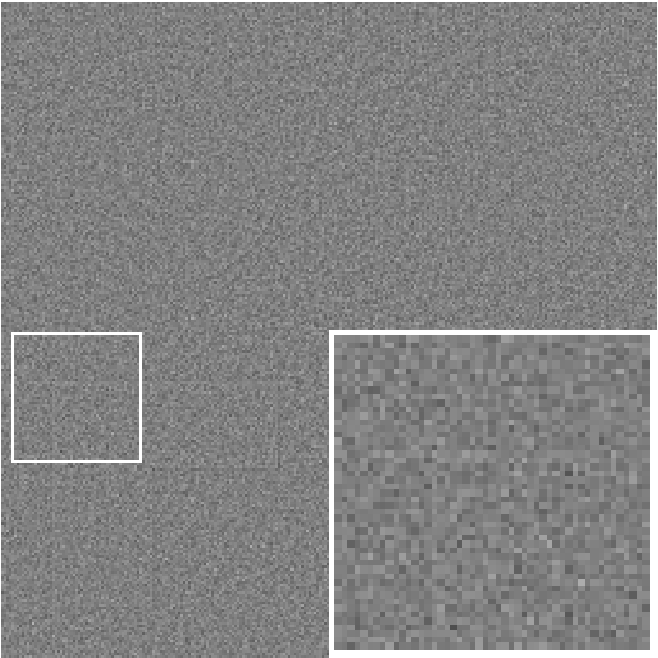}\\
        \vspace{0.02cm}
        \end{minipage}
        }
        \subfigure[\cite{wu2024variational}]{
        \begin{minipage}[b]{0.11\linewidth}
        \centering
        \includegraphics[scale=0.13]{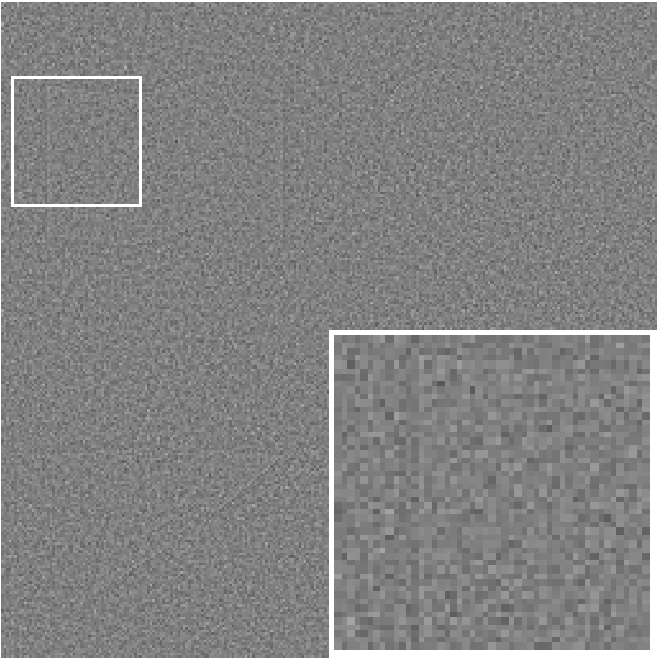}\\
        \vspace{0.02cm}
        \includegraphics[scale=0.13]{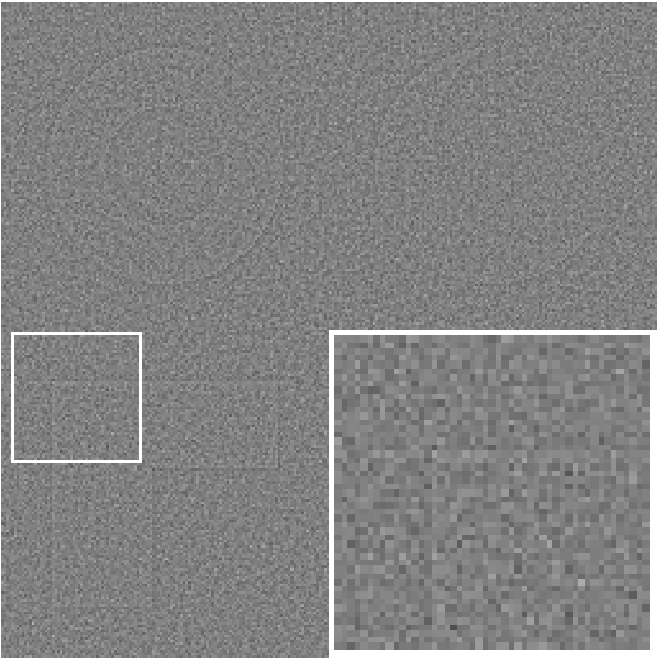}\\
        \vspace{0.02cm}
        \end{minipage}
        }
        \subfigure[\cite{su2026cone}]{
        \begin{minipage}[b]{0.11\linewidth}
        \centering
        \includegraphics[scale=0.13]{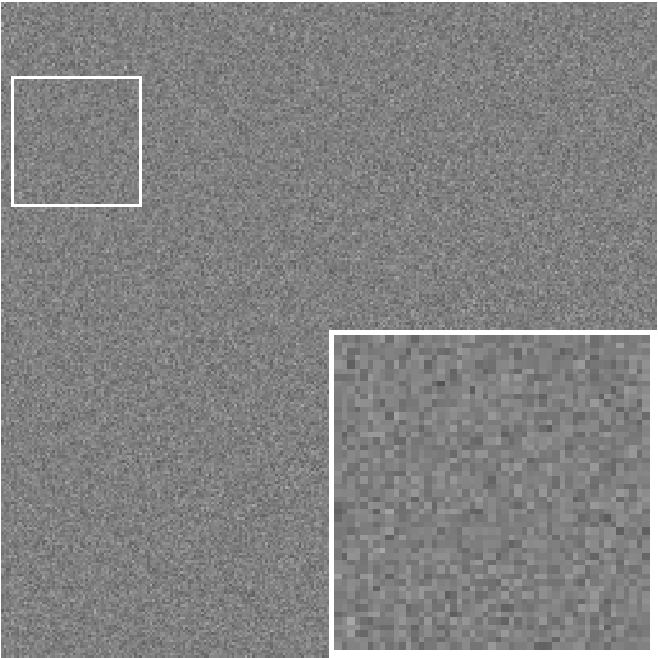}\\
        \vspace{0.02cm}
        \includegraphics[scale=0.13]{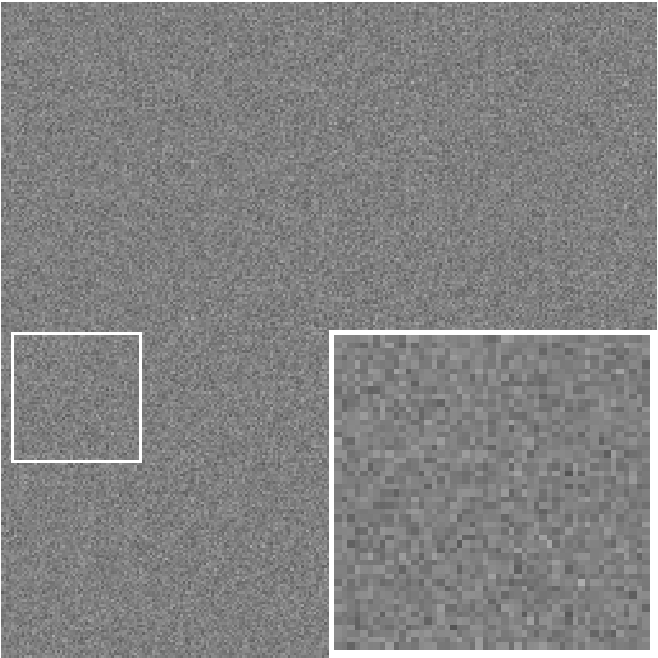}\\
        \vspace{0.02cm}
        \end{minipage}
        }
        \subfigure[Ours]{
        \begin{minipage}[b]{0.11\linewidth}
        \centering
        \includegraphics[scale=0.13]{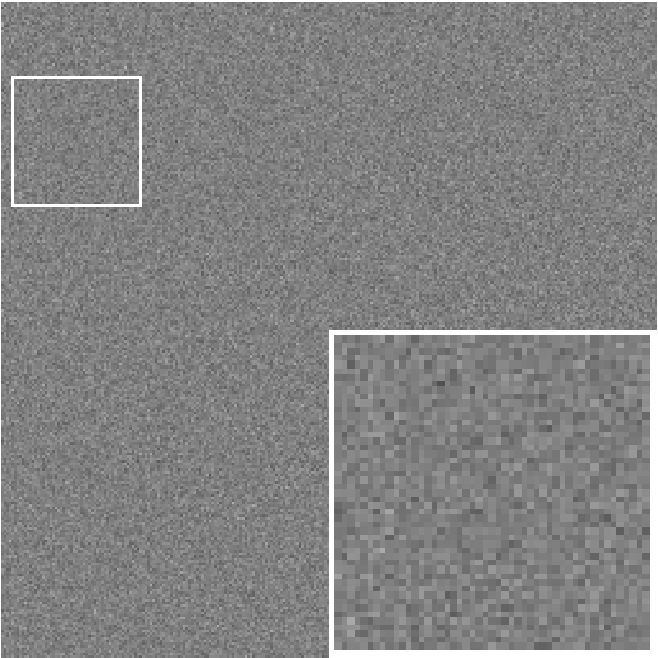}\\
        \vspace{0.02cm}
        \includegraphics[scale=0.13]{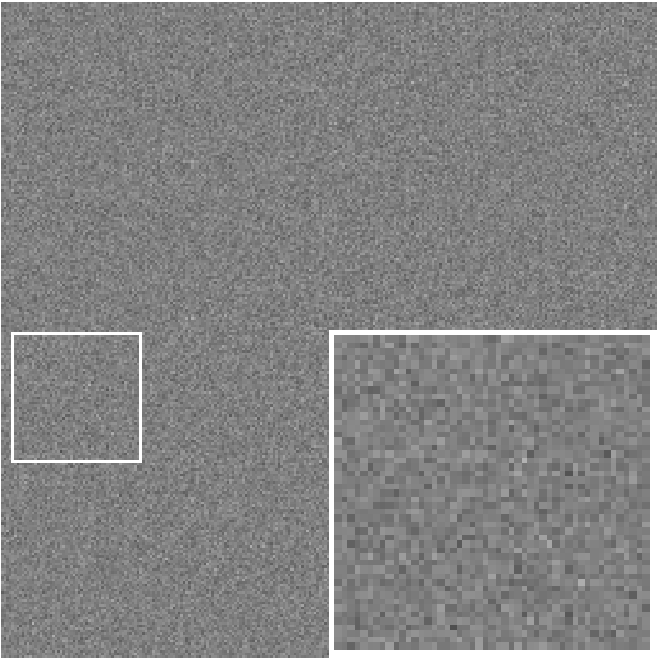}\\
        \vspace{0.02cm}
        \end{minipage}
        }
\caption{Residual images for Gaussian denoising on images “Synthetic1" and “Synthetic2". The residual images are computed by $f-u+0.5$.} \label{fig:boundary_preservation}
\end{figure}

\begin{figure}[t]
    \centering
    \subfigure[Noisy]{
        \begin{minipage}[b]{0.11\linewidth}
        \centering
        \includegraphics[scale=0.12]{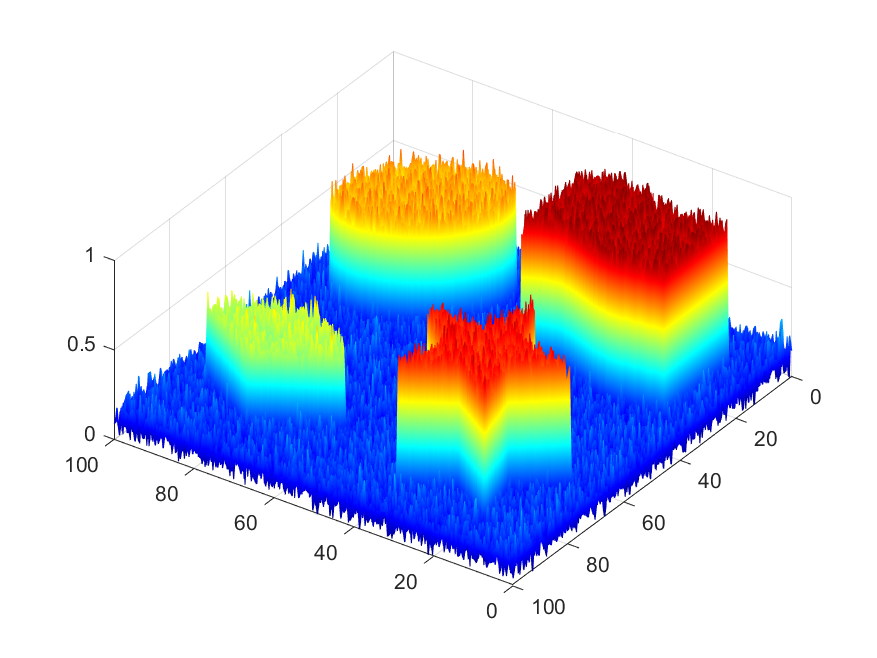}\\
        \vspace{0.02cm}
        \includegraphics[scale=0.12]{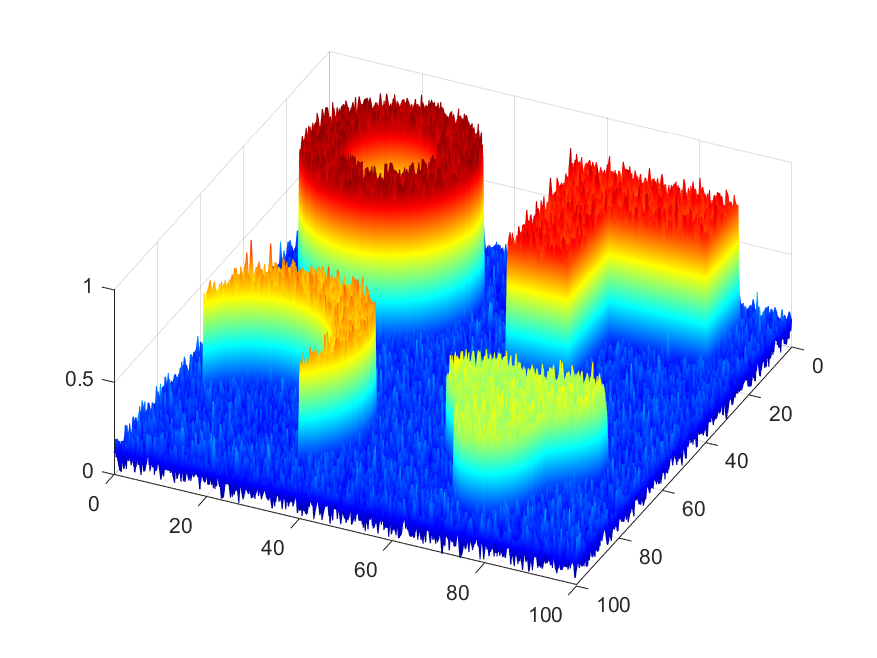}\\
        \vspace{0.02cm}
        \end{minipage}
        }
    \subfigure[\cite{WZ2012}]{
        \begin{minipage}[b]{0.11\linewidth}
        \centering
        \includegraphics[scale=0.12]{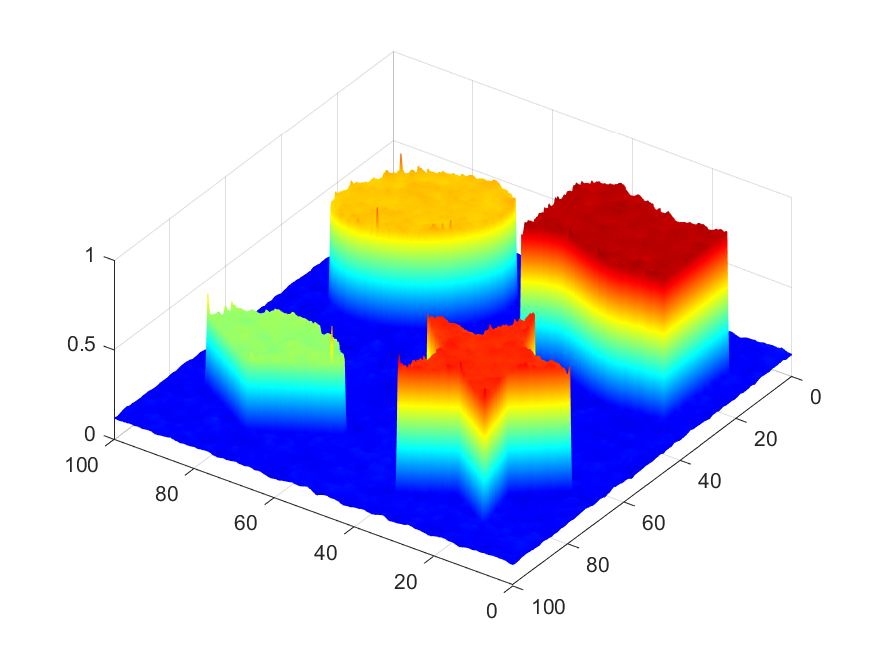}\\
        \vspace{0.02cm}
        \includegraphics[scale=0.12]{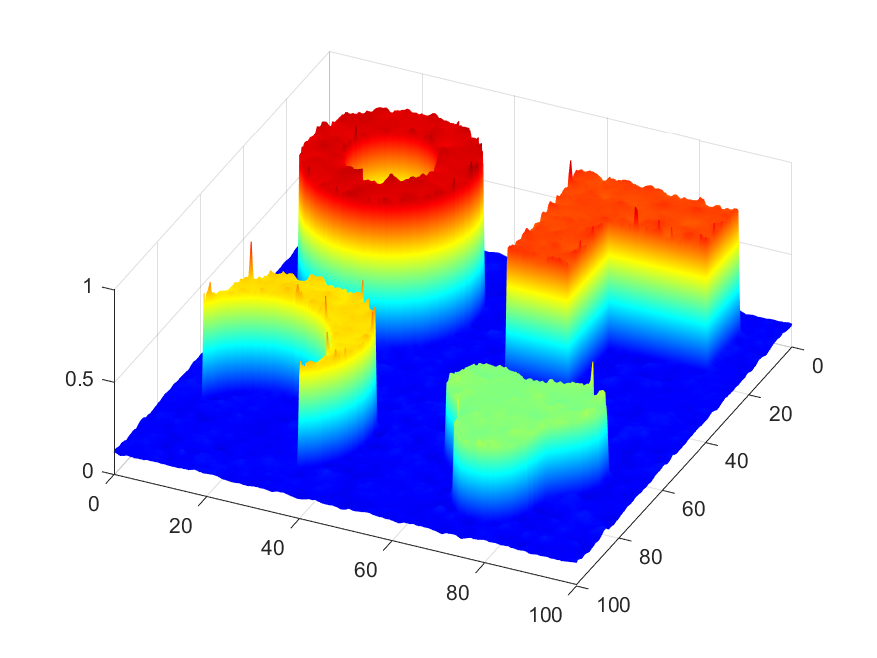}\\
        \vspace{0.02cm}
        \end{minipage}
        }
    \subfigure[\cite{CBL2016}]{
        \begin{minipage}[b]{0.11\linewidth}
        \centering
        \includegraphics[scale=0.12]{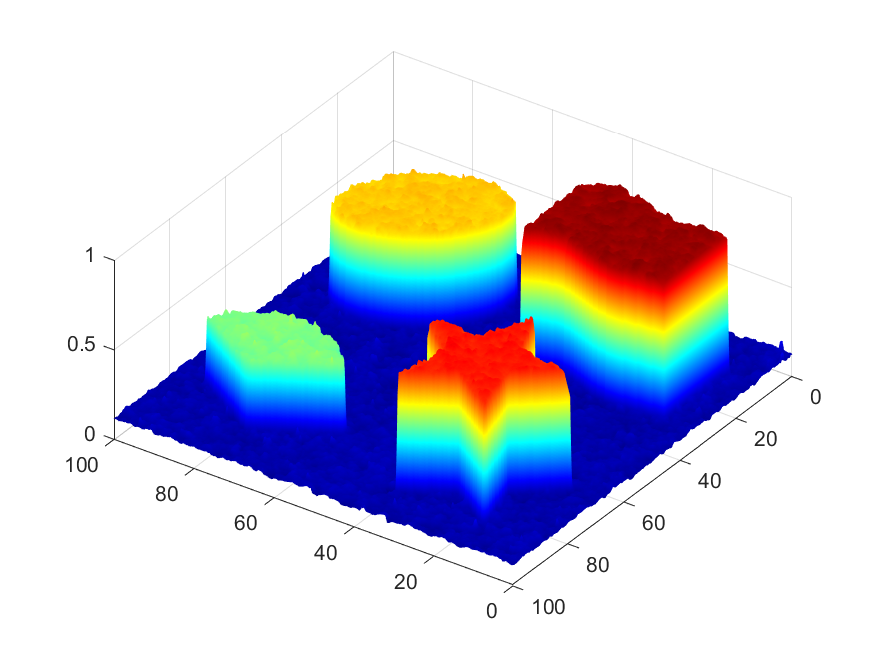}\\
        \vspace{0.02cm}
        \includegraphics[scale=0.12]{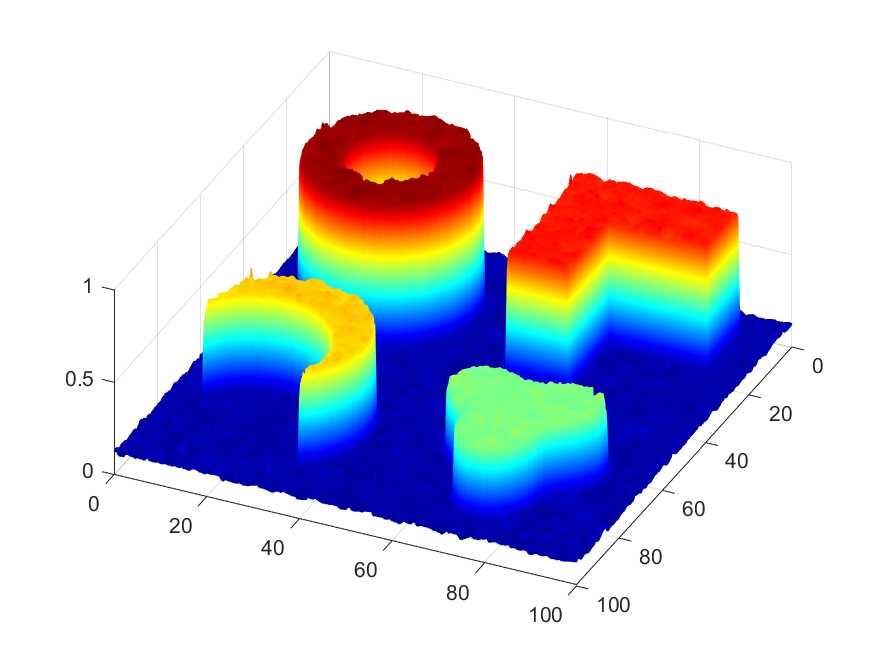}\\
        \vspace{0.02cm}
        \end{minipage}
        }
    \subfigure[\cite{QZ2022}]{
        \begin{minipage}[b]{0.11\linewidth}
        \centering
        \includegraphics[scale=0.12]{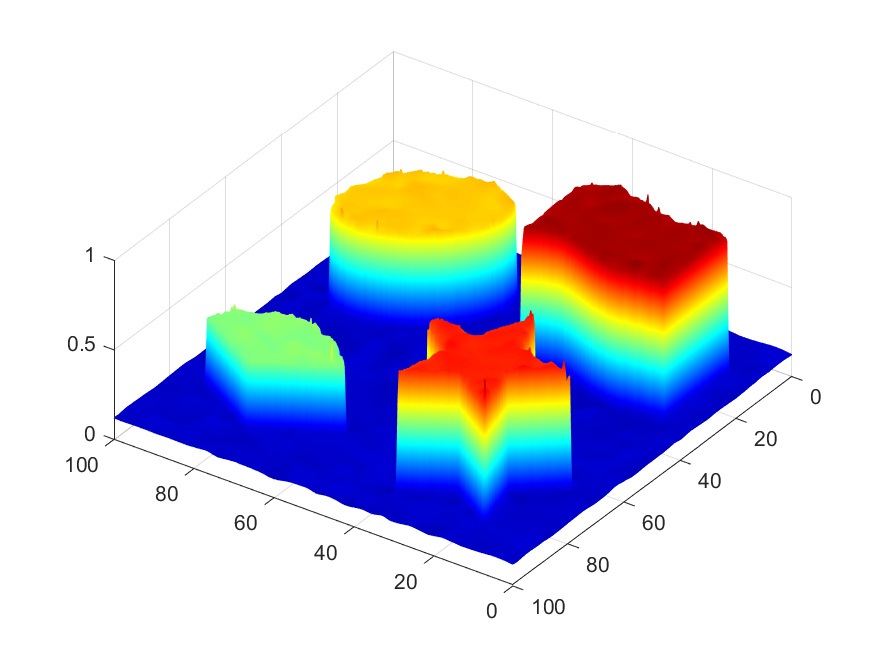}\\
        \vspace{0.02cm}
        \includegraphics[scale=0.12]{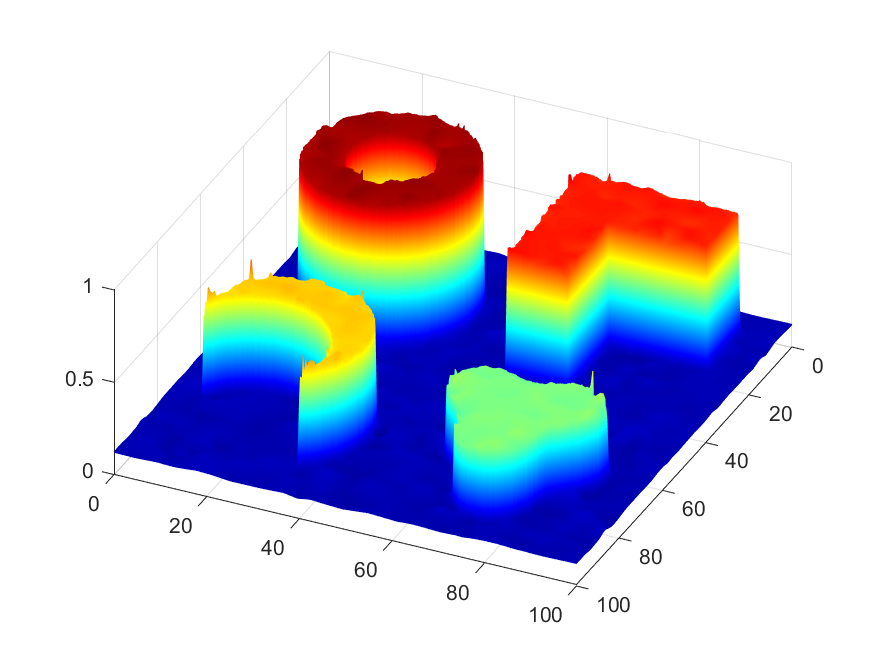}\\
        \vspace{0.02cm}
        \end{minipage}
        }
        \subfigure[\cite{wu2024variational}]{
        \begin{minipage}[b]{0.11\linewidth}
        \centering
        \includegraphics[scale=0.12]{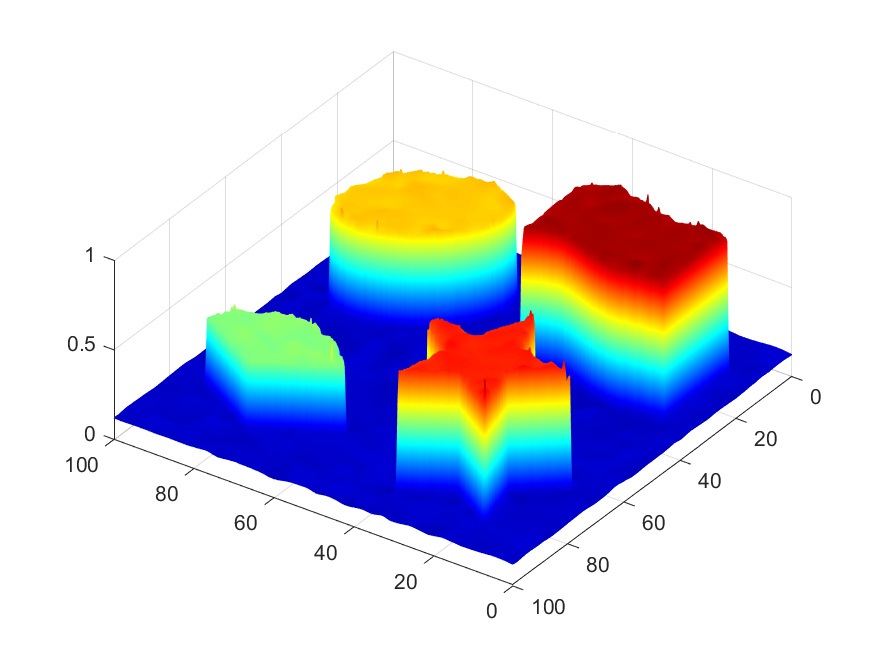}\\
        \vspace{0.02cm}
        \includegraphics[scale=0.12]{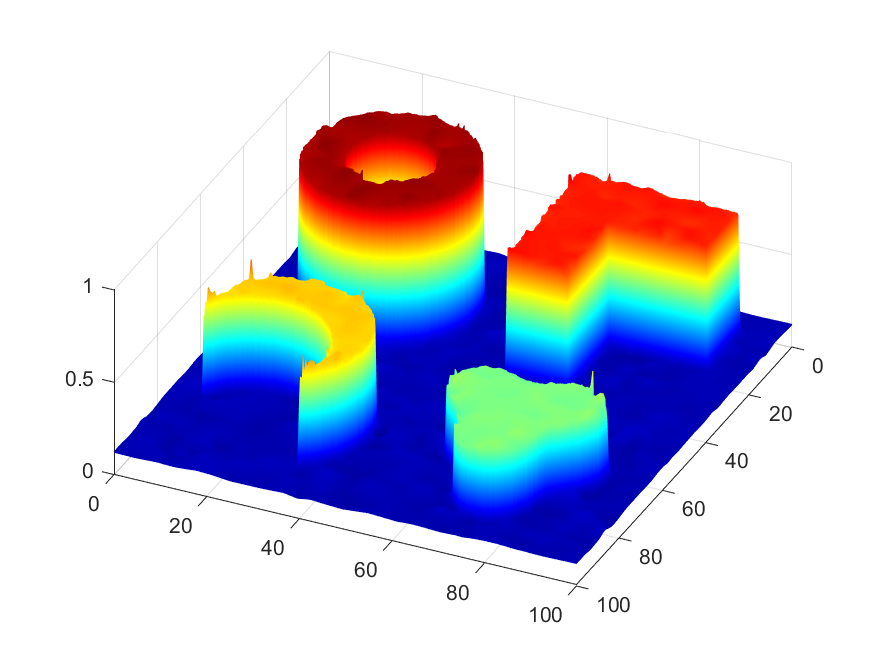}\\
        \vspace{0.02cm}
        \end{minipage}
        }
        \subfigure[\cite{su2026cone}]{
        \begin{minipage}[b]{0.11\linewidth}
        \centering
        \includegraphics[scale=0.12]{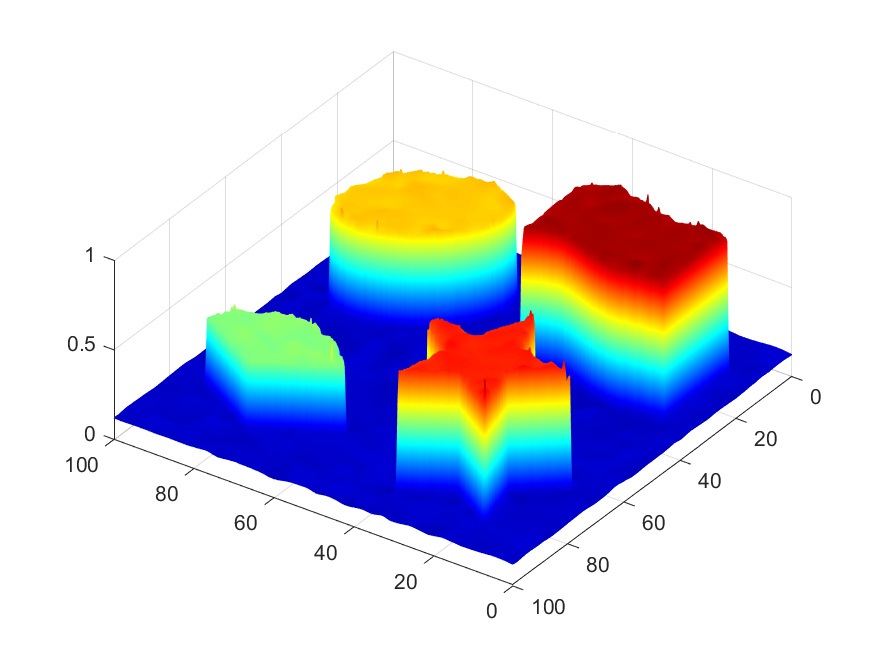}\\
        \vspace{0.02cm}
        \includegraphics[scale=0.12]{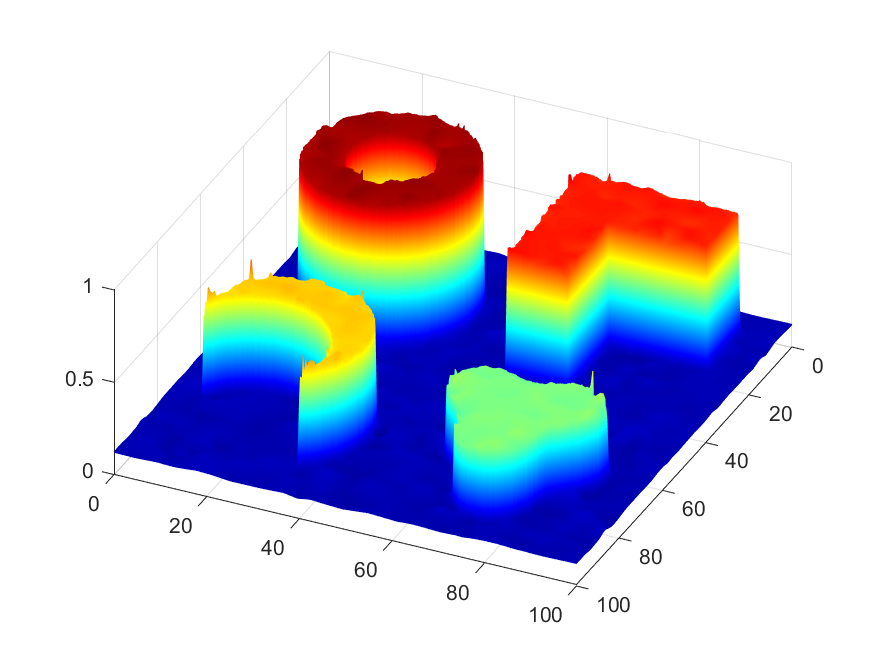}\\
        \vspace{0.02cm}
        \end{minipage}
        }
        \subfigure[Ours]{
        \begin{minipage}[b]{0.11\linewidth}
        \centering
        \includegraphics[scale=0.12]{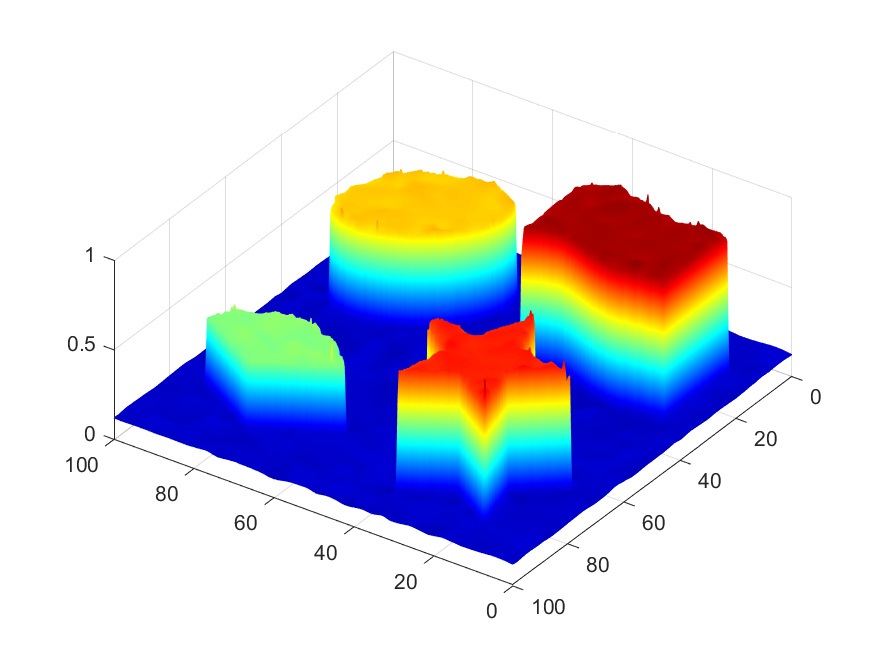}\\
        \vspace{0.02cm}
        \includegraphics[scale=0.12]{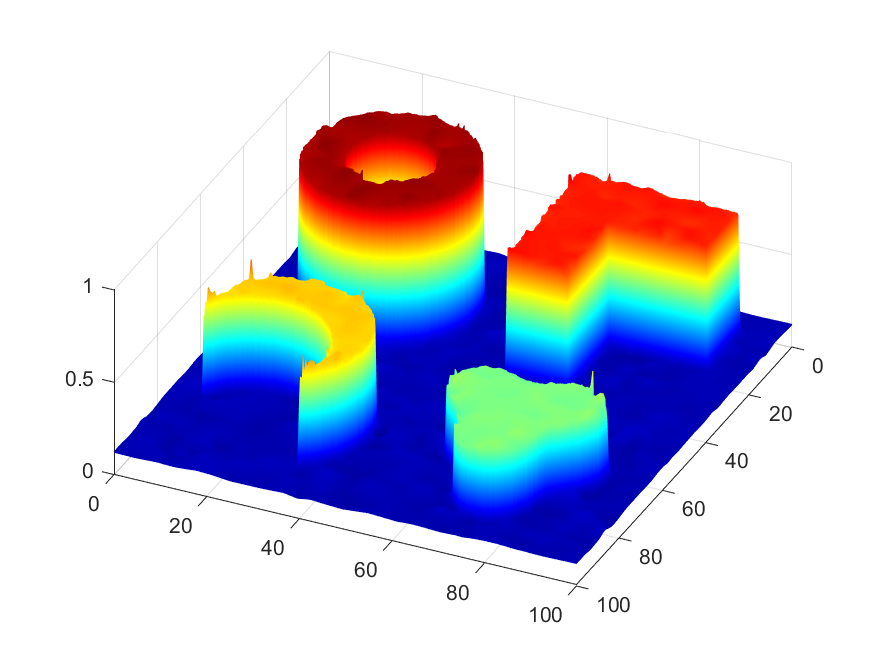}\\
        \vspace{0.02cm}
        \end{minipage}
        }
\caption{Image surface comparison for Gaussian denoising on images “Synthetic1" and “Synthetic2".} \label{figure_surface}
\end{figure}

%\begin{table}[tbp]
%\centering
%\footnotesize{
%\caption{PSNR and SSIM results for Gaussian denoising of “Synthetic1" and “Synthetic2".} \label{tab:boundary_preservation}
%\setlength{\tabcolsep}{0.5mm}{
%\vspace{0.1cm}
%\begin{tabular}{ccccccccccccc}
%\hline
%           & \multicolumn{2}{c}{\cite{WZ2012}} & \multicolumn{2}{c}{\cite{CBL2016}} & \multicolumn{2}{c}{\cite{QZ2022}} & \multicolumn{2}{c}{\cite{wu2024variational}} & \multicolumn{2}{c}{\cite{su2026cone}} & \multicolumn{2}{c}{Ours}         \\
%           \cmidrule(lr){2-3} \cmidrule(lr){4-5} \cmidrule(lr){6-7} \cmidrule(lr){8-9}\cmidrule(lr){10-11}\cmidrule(lr){12-13}
%           & PSNR           & SSIM             & PSNR            & SSIM             & PSNR            & SSIM            & PSNR                 & SSIM                  & PSNR              & SSIM              & PSNR           & SSIM            \\ \hline
%Synthetic1 & 42.67          & 0.9920          & 38.34           & 0.9707           & 41.36           & 0.9951          & 38.80                & 0.9788                & 44.52             & 0.9981            & \textbf{44.64} & \textbf{0.9986} \\
%Synthetic2 & 40.74          & 0.9889           & 38.24           & 0.9717           & 40.55           & 0.9944          & 38.23                & 0.9788                & 45.63             & 0.9984            & \textbf{45.64} & \textbf{0.9992} \\ \hline
%\end{tabular}}}
%\end{table}
\begin{table}[tbp]
\centering
\small{
\caption{PSNR values for Gaussian denoising of “Synthetic1" and “Synthetic2".} \label{tab:boundary_preservation}
\setlength{\tabcolsep}{0.6mm}{
\vspace{0.1cm}
\begin{tabular}{ccccccc}
\hline
           & MC \cite{WZ2012} & GC \cite{CBL2016} & WM \cite{QZ2022} & IOS--TSGV \cite{wu2024variational} & MMAMM--TSGV \cite{su2026cone} & Ours           \\ \hline
Synthetic1 & 42.67       & 38.34        & 41.36       & 38.80                  & 44.52           & \textbf{44.64} \\
Synthetic2 & 40.74       & 38.24        & 40.55       & 38.23                  & 45.63           & \textbf{45.64} \\ \hline
\end{tabular}}}
\end{table}

In this subsection, we evaluate the boundary preserving capability
of the proposed method using the synthetic images “Synthetic1"
and “Synthetic2". We compare our method with the mean curvature model (MC) \cite{WZ2012}, Gaussian curvature model (GC) \cite{CBL2016}, Weingarten--map based model (WM) \cite{QZ2022}, and Total scaled-gradient variation model (TSGV) \cite{wu2024variational,su2026cone}. Edge-, corner-, and contrast preservation properties of these models have been analyzed under suitable assumptions.

Figure~\ref{fig:boundary_preservation} presents the noisy images and the corresponding residual images. As shown in the enlarged regions, the residual images produced by the proposed method contain less visible boundary and corner information than the most competing methods. This observation indicates that our method more effectively separates noise from geometrically significant structures and better preserves object boundaries and corners.

We further compare the recovery of smooth image regions through the
restored image surfaces shown in Figure~\ref{figure_surface}. The surfaces reconstructed by the proposed method contain fewer oscillations in homogeneous regions, while the principal boundary structures remain well preserved.

Table~\ref{tab:boundary_preservation} reports the corresponding
PSNR values. Our method attains the highest PSNR on both test images. These quantitative results are consistent with the residual image comparison in Figure~\ref{fig:boundary_preservation}.

\subsection{Effect of Anderson acceleration and energy decay}

\begin{figure}[tbp]
    \centering
    \subfigure[Airplane]{
    \includegraphics[scale=0.35]{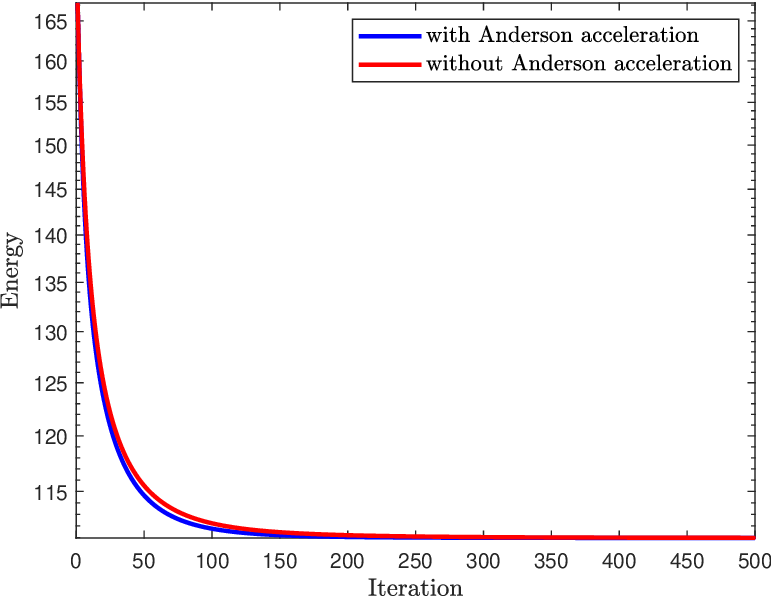}
    }
    \subfigure[Boat]{
    \includegraphics[scale=0.35]{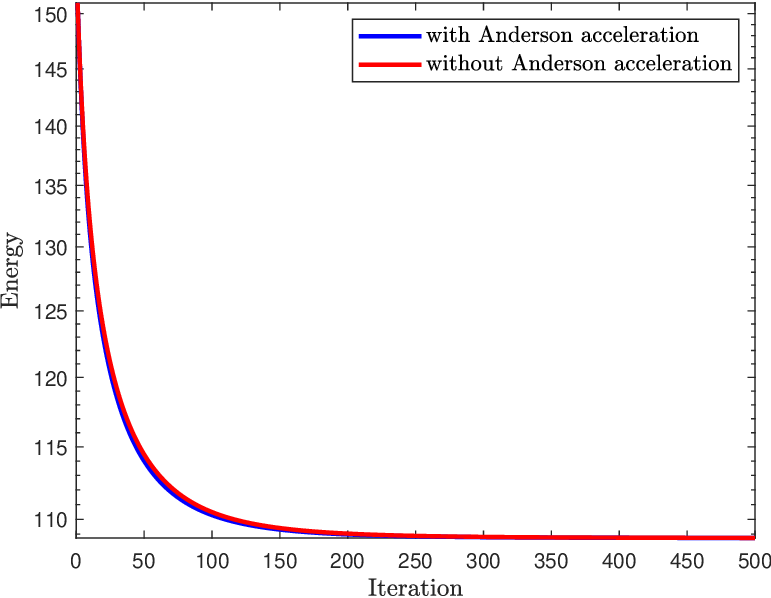}
    }
\caption{Evolution of the energy for “Airplane" and “Boat", with and without Anderson acceleration.} \label{fig:energy_decay}
\end{figure}

We use the images “Airplane" and “Boat" to examine the numerical decay of the energy and the effect of depth-one Anderson acceleration. We first examine the effect of depth-one Anderson acceleration. The accelerated and nonaccelerated variants are initialized identically and use the same model and algorithmic parameters. For “Airplane", the accelerated method terminates after \(123\) iterations and yields a PSNR of \(32.47\) dB, whereas the nonaccelerated method requires \(136\) iterations and yields a PSNR of \(32.37\) dB. For “Boat", the corresponding results are \(31.39\) dB after \(149\) iterations with Anderson acceleration and \(31.36\) dB after \(154\) iterations without acceleration. These
results indicate that the depth-one Anderson acceleration reduces the
number of iterations while maintaining, and in these two examples
slightly improving, the reconstruction quality.

We next investigate the decay of the objective energy.
To provide a direct comparison of the two variants, both variants are run for $500$ iterations. Figure~\ref{fig:energy_decay} shows the energy histories generated by the proposed method with and without Anderson acceleration. Both variants exhibit monotone energy decay and approach approximately the same terminal energy value, while the accelerated variant reaches comparable energy levels in fewer iterations.

\subsection{Comparison with representative methods}

\begin{table}[tbp]
\centering
\caption{Quantitative comparison of the WM, EE, and TSGV
methods with the proposed AAMM--GMV method for Gaussian denoising.}
\label{table_comparison}
\scriptsize{
\setlength{\tabcolsep}{0.1mm}
\begin{tabular}{lccccccccccccccc}
\hline
          & \multicolumn{3}{c}{WM \cite{QZ2022}} & \multicolumn{3}{c}{EE \cite{liu2024fast}} & \multicolumn{3}{c}{IOS--TSGV \cite{wu2024variational}} & \multicolumn{3}{c}{MMAMM--TSGV \cite{su2026cone}} & \multicolumn{3}{c}{Ours}         \\
          \cmidrule(lr){2-4} \cmidrule(lr){5-7} \cmidrule(lr){8-10} \cmidrule(lr){11-13}\cmidrule(lr){14-16}
          & PSNR            & $\#$Iter & Time & PSNR       & $\#$Iter      & Time      & PSNR         & $\#$Iter        & Time        & PSNR             & $\#$Iter   & Time  & PSNR           & $\#$Iter & Time \\ \hline
Airplane  & 32.35           & 194      & 1.4  & 31.70      & 447           & 4.4       & 32.29        & 26              & 0.2         & 32.30            & 240        & 1.1   & \textbf{32.47} & 123      & 0.8  \\
Boat      & 31.22           & 192      & 1.4  & 30.80      & 500           & 4.8       & 31.37        & 28              & 0.2         & 31.34            & 296        & 1.4   & \textbf{31.39} & 149      & 1.0  \\
Cameraman & \textbf{31.69}  & 214      & 1.6  & 31.36      & 500           & 4.9       & 31.60        & 42              & 0.3         & 31.61            & 304        & 1.4   & 31.64          & 79       & 0.6  \\
Couple    & 30.79           & 198      & 1.4  & 30.47      & 467           & 4.6       & 31.01        & 28              & 0.2         & \textbf{31.03}   & 311        & 2.0   & \textbf{31.03} & 190      & 1.3  \\
Monarch   & 32.49           & 300      & 2.1  & 31.65      & 500           & 4.8       & 32.50        & 22              & 0.2         & 32.43            & 286        & 1.7   & \textbf{32.56} & 103      & 0.7  \\
Parrot    & 31.75           & 213      & 1.5  & 31.25      & 500           & 4.8       & 31.71        & 38              & 0.3         & 31.71            & 301        & 2.3   & \textbf{31.79} & 96       & 0.7  \\
Peppers   & 32.75           & 183      & 1.4  & 31.69      & 500           & 4.7       & 32.74        & 28              & 0.2         & 32.86            & 275        & 1.7   & \textbf{32.89} & 92       & 0.6  \\
Starfish  & 31.61           & 300      & 2.1  & 30.72      & 500           & 4.7       & 31.63        & 27              & 0.2         & 31.63            & 277        & 1.7   & \textbf{31.79} & 112      & 0.8  \\ \hline
\end{tabular}
}
\end{table}

We compare the proposed method with representative geometric regularization methods, including the WM \cite{QZ2022}, the Euler elastica model (EE) \cite{liu2024fast}, and the TSGV models \cite{wu2024variational,su2026cone}. Table~\ref{table_comparison} reports the PSNR values, number of iterations, and running time. Our method attains the highest PSNR on six of the eight test images, the same highest PSNR as MMAMM--TSGV on “Couple", and the second-highest PSNR on “Cameraman". It also requires fewer iterations and shorter
running times than WM, EE, and MMAMM--TSGV. Although IOS--TSGV is
faster, AAMM--GMV yields a higher PSNR on every test image.

\subsection{Impulsive noise removal}

\begin{table}[tbp]
\centering
\caption{Quantitative comparison of the IOS--TSGV with the proposed AAMM--GMV method for impulsive denoising.}
\label{table_comparison_L1}
\small{
\setlength{\tabcolsep}{0.2mm}
\begin{tabular}{cccccccccc}
\hline
                                                   &          & Airplane       & Boat           & Cameraman      & Couple         & Monarch        & Parrot         & Peppers        & Starfish       \\ \cline{2-10}
\multirow{3}{*}{IOS--TSGV \cite{wu2024variational}} & PSNR     & 36.44          & 35.70          & 33.02          & 36.05          & 36.34          & 33.04          & \textbf{34.13} & 36.16          \\
                                                   & $\#$Iter & 52             & 46             & 91             & 42             & 48             & 70             & 95             & 45             \\
                                                   & Time     & 0.5            & 0.4            & 0.8            & 0.4            & 0.5            & 0.7            & 0.9            & 0.4            \\ \hline
\multirow{3}{*}{Ours}                              & PSNR     & \textbf{37.08} & \textbf{36.19} & \textbf{33.17} & \textbf{36.45} & \textbf{36.51} & \textbf{33.27} & 31.14          & \textbf{36.71} \\
                                                   & $\#$Iter & 300            & 128            & 78             & 300            & 300            & 214            & 300            & 300            \\
                                                   & Time     & 1.9            & 0.9            & 0.6            & 1.9            & 1.9            & 1.3            & 1.9            & 1.9            \\ \hline
\end{tabular}
}
\end{table}

Finally, we evaluate the proposed method for impulsive noise removal
and compare it with the IOS--TSGV in \cite{wu2024variational}. In our method, while keeping the other parameter settings unchanged, we set $\mathrm{Tol}=10^{-5}$. Table~\ref{table_comparison_L1} reports the corresponding quantitative results. The results show that the proposed method achieves the highest PSNR value on all test images except “Peppers". These results demonstrate the competitive performance of the proposed method for removing impulsive noise.

\section{Conclusion}\label{sec:conclusion}

In this paper, we proposed a Gauss-map variation model that exploits
the tangent plane geometry of the scaled image graph and analyzed its
behavior near general $C^2$ and piecewise $C^2$ boundaries, using
differential geometric tools. To solve the resulting model, we developed an Anderson-accelerated majorization--minimization algorithm based on a bilinear decomposition. The normal field subproblem admits an explicit MM update, while a depth-one Anderson acceleration improves its practical efficiency without sacrificing energy descent. We proved that the generated sequence has finite length and converges to a critical point of the constrained penalized objective. Numerical experiments demonstrated the boundary preserving capability of the proposed model, the monotone decay of the objective energy, and its competitive performance for Gaussian and impulsive noise removal.

\section*{Acknowledgments}
The author would like to thank the researchers who kindly provided
their implementations or made their source codes publicly available.

%-------------------------------------------------------------------------------------

\bibliographystyle{siamplain}
\bibliography{reference}

@book{do2016differential,
  title={Differential geometry of curves and surfaces: revised and updated second edition},
  author={Do Carmo, Manfredo P},
  year={2016},
  publisher={Courier Dover Publications}
}

@book{lee2018introduction,
  title={Introduction to Riemannian manifolds},
  author={Lee, John M},
  volume={2},
  year={2018},
  publisher={Springer Cham}
}

@book{lee2006riemannian,
  title={Riemannian manifolds: an introduction to curvature},
  author={Lee, John M},
  year={2006},
  publisher={Springer Science \& Business Media}
}

@article{LR1992,
  title={Nonlinear total variation based noise removal algorithms},
  author={Rudin, Leonid I and Osher, Stanley and Fatemi, Emad},
  journal={Physica D},
  volume={60},
  number={1-4},
  pages={259--268},
  year={1992},
  publisher={Elsevier}
}

@article{AC1997,
  title={Image recovery via total variation minimization and related problems},
  author={Chambolle, Antonin and Lions, Pierre-Louis},
  journal={Numerische Mathematik},
  volume={76},
  number={2},
  pages={167--188},
  year={1997},
  publisher={Springer}
}

@book{YM2002,
  title={Oscillating patterns in image processing and nonlinear evolution equations: the fifteenth Dean Jacqueline B. Lewis memorial lectures},
  author={Meyer, Yves},
  volume={22},
  year={2001},
  publisher={American Mathematical Society}
}

@article{YLY2000,
  title={Fourth-order partial differential equations for noise removal},
  author={You, Y-L and Kaveh, Mostafa},
  journal={IEEE Transactions on Image Processing},
  volume={9},
  number={10},
  pages={1723--1730},
  year={2000},
  publisher={IEEE}
}

@article{ML2003,
  title={Noise removal using fourth-order partial differential equation with applications to medical magnetic resonance images in space and time},
  author={Lysaker, Marius and Lundervold, Arvid and Tai, Xue-Cheng},
  journal={IEEE Transactions on Image processing},
  volume={12},
  number={12},
  pages={1579--1590},
  year={2003},
  publisher={IEEE}
}

@article{WZ2012,
  title={Image denoising using mean curvature of image surface},
  author={Zhu, Wei and Chan, Tony},
  journal={SIAM Journal on Imaging Sciences},
  volume={5},
  number={1},
  pages={1--32},
  year={2012},
  publisher={SIAM}
}

@article{CBL2016,
  title={Image denoising using the gaussian curvature of the image surface},
  author={Brito-Loeza, Carlos and Chen, Ke and Uc-Cetina, Victor},
  journal={Numerical Methods for Partial Differential Equations},
  volume={32},
  number={3},
  pages={1066--1089},
  year={2016},
  publisher={Wiley Online Library}
}

@article{HL2022,
  title={An Operator-Splitting Method for the Gaussian Curvature Regularization Model with Applications to Surface Smoothing and Imaging},
  author={Liu, Hao and Tai, Xue-Cheng and Glowinski, Roland},
  journal={SIAM Journal on Scientific Computing},
  volume={44},
  number={2},
  pages={A935--A963},
  year={2022},
  publisher={SIAM}
}

@article{KB2010,
  title={Total generalized variation},
  author={Bredies, Kristian and Kunisch, Karl and Pock, Thomas},
  journal={SIAM Journal on Imaging Sciences},
  volume={3},
  number={3},
  pages={492--526},
  year={2010},
  publisher={SIAM}
}

@article{QZ2022,
  title={Spatially adapted first and second order regularization for image reconstruction: from an image surface perspective},
  author={Zhong, Qiuxiang and Liu, Ryan Wen and Duan, Yuping},
  journal={Journal of Scientific Computing},
  volume={92},
  number={2},
  pages={1--36},
  year={2022},
  publisher={Springer}
}

@article{WZ2013,
  title={Augmented Lagrangian method for a mean curvature based image denoising model},
  author={Zhu, Wei and Tai, Xue-Cheng and Chan, Tony},
  journal={Inverse Problems \& Imaging},
  volume={7},
  number={4},
  pages={1409-1432},
  year={2013},
  publisher={American Institute of Mathematical Sciences}
}

@article{strong2003edge,
  title={Edge-preserving and scale-dependent properties of total variation regularization},
  author={Strong, David and Chan, Tony},
  journal={Inverse Problems},
  volume={19},
  number={6},
  pages={S165},
  year={2003},
  publisher={IOP Publishing}
}

@article{bellettini2002total,
  title={The total variation flow in {$\mathbb{R}^n$}},
  author={Bellettini, Giovanni and Caselles, Vicent and Novaga, Matteo},
  journal={Journal of Differential Equations},
  volume={184},
  number={2},
  pages={475--525},
  year={2002},
  publisher={Elsevier}
}

@article{dobson1996recovery,
  title={Recovery of blocky images from noisy and blurred data},
  author={Dobson, David C and Santosa, Fadil},
  journal={SIAM Journal on Applied Mathematics},
  volume={56},
  number={4},
  pages={1181--1198},
  year={1996},
  publisher={SIAM}
}

@article{wu2024variational,
  title={On variational image denoising: A total scaled-gradient variation regularization framework and an accelerated operator-splitting solver},
  author={Wu, Chunlin and Su, Haibin and Zhang, Daoping},
  journal={Communications in Mathematical Sciences},
  volume={23},
  number={2},
  pages={329--358},
  year={2024},
  publisher={International Press of Boston}
}

@article{liu2024fast,
  title={A fast minimization algorithm for the Euler elastica model based on a bilinear decomposition},
  author={Liu, Zhifang and Sun, Baochen and Tai, Xue-Cheng and Wang, Qi and Chang, Huibin},
  journal={SIAM Journal on Scientific Computing},
  volume={46},
  number={1},
  pages={A290--A312},
  year={2024},
  publisher={SIAM}
}

@article{attouch2013convergence,
  title={Convergence of descent methods for semi-algebraic and tame problems: proximal algorithms, forward--backward splitting, and regularized Gauss--Seidel methods},
  author={Attouch, Hedy and Bolte, J{\'e}r{\^o}me and Svaiter, Benar Fux},
  journal={Mathematical programming},
  volume={137},
  number={1},
  pages={91--129},
  year={2013},
  publisher={Springer}
}

@article{bolte2014proximal,
  title={Proximal alternating linearized minimization for nonconvex and nonsmooth problems},
  author={Bolte, J{\'e}r{\^o}me and Sabach, Shoham and Teboulle, Marc},
  journal={Mathematical Programming},
  volume={146},
  number={1},
  pages={459--494},
  year={2014},
  publisher={Springer}
}

@article{su2026cone,
  title={A globally convergent algorithm via cone-constrained bilinear decomposition for total scaled-gradient variation in image restoration},
  author={Su, Haibin and Wu, Chunlin and Chang, Huibin and Liu, Zhifang},
  journal={arXiv preprint arXiv:2609.00036},
  year={2026}
}

@article{Shen2003euler,
    author = {Shen, Jianhong and Kang, Sung Ha and Chan, Tony F.},
    title = {Euler's elastica and curvature-based inpainting},
    journal = {SIAM Journal on Applied Mathematics},
    volume = {63},
    number = {2},
    pages = {564-592},
    year = {2003}
}

@article{Tai2011Afast,
    author = {Tai, Xue-Cheng and Hahn, Jooyoung and Chung, Ginmo Jason},
    title = {A fast algorithm for Euler's elastica model using augmented Lagrangian method},
    journal = {SIAM Journal on Imaging Sciences},
    volume = {4},
    number = {1},
    pages = {313-344},
    year = {2011}
}

@article{Zhu2013Image,
  title={Image segmentation using Euler’s elastica as the regularization},
  author={Zhu, Wei and Tai, Xue-Cheng and Chan, Tony},
  journal={Journal of Scientific Computing},
  volume={57},
  pages={414--438},
  year={2013},
  publisher={Springer}
}

@article{Deng2019Anew,
  title={A new operator splitting method for the Euler elastica model for image smoothing},
  author={Deng, Liang-Jian and Glowinski, Roland and Tai, Xue-Cheng},
  journal={SIAM Journal on Imaging Sciences},
  volume={12},
  number={2},
  pages={1190--1230},
  year={2019},
  publisher={SIAM}
}

@article{su2026accelerated,
  title={An accelerated operator splitting algorithm for solving Euler's elastica-based image inpainting and segmentation models},
  author={Su, Haibin and Wu, Chunlin and Duan, Yuping and Xue, Yunhua},
  journal={Communications in Computational Physics},
  volume={40},
  number={5},
  pages={1567--1591},
  year={2026}
}

@article{zhao2026image,
  title={Image denoising via $ L^2$-norm of anisotropic mean curvature of image surface},
  author={Zhao, Yuanyuan and Song, Xianhua and Guo, Zhichang},
  journal={Journal of Mathematical Imaging and Vision},
  volume={68},
  number={3},
  pages={29},
  year={2026},
  publisher={Springer}
}

@article{zhu2020image,
  title={Image denoising using $L^p$-norm of mean curvature of image surface},
  author={Zhu, Wei},
  journal={Journal of Scientific Computing},
  volume={83},
  number={2},
  pages={32},
  year={2020},
  publisher={Springer}
}

@article{zhang2025edge,
  title={Edge-adaptive curvature-based regularization for robust blind image restoration},
  author={Zhang, Tingting and Lu, Jiawei and Wu, Caiying and Duan, Yuping and Liu, Jun and Zeng, Tieyong and Jin, Qiyu and Chen, Guoqing and Morel, Jean-Michel and Wu, Boying and others},
  journal={Physics Letters A},
  pages={131300},
  year={2025},
  publisher={Elsevier}
}

@article{zhong2021image,
  title={Image reconstruction by minimizing curvatures on image surface},
  author={Zhong, Qiuxiang and Yin, Ke and Duan, Yuping},
  journal={Journal of Mathematical Imaging and Vision},
  volume={63},
  number={1},
  pages={30--55},
  year={2021},
  publisher={Springer}
}

@article{wu2025total,
  title={Total curvature-driven blind image deblurring},
  author={Wu, Caiying and Zhang, Lulu and Zhang, Tingting and Lu, Jiawei and Chen, Guoqing and Liu, Jun and Jin, Qiyu and Zeng, Tieyong},
  journal={Journal of Computational and Applied Mathematics},
  pages={117089},
  year={2025},
  publisher={Elsevier}
}

@article{lu2026total,
  title={Total normal curvature regularization and its minimization for surface and image smoothing},
  author={Lu, Tianle and Chen, Ke and Duan, Yuping},
  journal={SIAM Journal on Imaging Sciences},
  volume={19},
  number={2},
  pages={879--912},
  year={2026},
  publisher={SIAM}
}

@article{liu2025minimizing,
  title={Minimizing mean curvature of image surface via bilinear decomposition for image reconstruction},
  author={Liu, Zhifang and Chang, Huibin and Gao, Yiming},
  journal={Journal of Mathematical Imaging and Vision},
  volume={67},
  number={6},
  pages={55},
  year={2025},
  publisher={Springer}
}
\end{document}